\documentclass[12pt]{article}
\usepackage[margin=1in]{geometry}
\usepackage{setspace}
\usepackage[acronym]{glossaries}
\usepackage{booktabs}
\usepackage{graphicx}
\usepackage{subcaption}
\graphicspath{{Figures/}}

\RequirePackage{tgtermes}
\RequirePackage{newtxtext}
\RequirePackage{newtxmath}
\RequirePackage{bm}
\RequirePackage{endnotes}

\usepackage{algorithm}
\usepackage{algpseudocode}
\usepackage{tikz}
\usepackage{mathtools}
\usetikzlibrary{arrows.meta, positioning, shapes.geometric}
\usepackage{placeins}

\usepackage{natbib}
\usepackage{tabularx}
\usepackage{booktabs} 
\usepackage{multirow} 
\newcolumntype{Y}{>{\raggedright\arraybackslash}X}

\newcommand{\Wtwo}{\mathrm{W}_2}

\makeglossaries 
\newacronym{qaly}{QALY}{quality-adjusted life year}
\newacronym{ASCVD}{ASCVD}{atherosclerotic cardiovascular disease}
\newacronym{BP}{BP}{blood pressure}
\newacronym{KL}{KL}{Kullback-Leibler}
\newacronym{JS}{JS}{Jensen-Shannon}
\newacronym{MDP}{MDP}{Markov decision process}
\newacronym{DRL}{DRL}{distributional reinforcement learning}
\newacronym{BDRL}{BDRL}{boosted distributional reinforcement learning}
\newacronym{RL}{RL}{Reinforcement Learning}

\newtheorem{theorem}{Theorem}
\newtheorem{proposition}{Proposition}
\newtheorem{property}{Property}

\providecommand{\qedsymbol}{$\square$}
\newenvironment{proof}[1][Proof]
  {\par\noindent\textit{#1. }\ignorespaces}
  {\hfill\qedsymbol\par}

\title{Boosted Distributional Reinforcement Learning: Analysis and Healthcare Applications}
\author{
Zequn Chen\\
Thayer School of Engineering, Dartmouth College\\
\texttt{zequn.chen.th@dartmouth.edu}
\and
Wesley J. Marrero\\
Thayer School of Engineering, Dartmouth College\\
\texttt{wesley.marrero@dartmouth.edu}
}

\date{}

\begin{document}

\maketitle

\begin{abstract}
Researchers and practitioners are increasingly considering reinforcement learning to optimize decisions in complex domains like robotics and healthcare. To date, these efforts have largely utilized expectation-based learning. However, relying on expectation-focused objectives may be insufficient for making consistent decisions in highly uncertain situations involving multiple heterogeneous groups. While distributional reinforcement learning algorithms have been introduced to model the full distributions of outcomes, they can yield large discrepancies in realized benefits among comparable agents. This challenge is particularly acute in healthcare settings, where physicians (controllers) must manage multiple patients (subordinate agents) with uncertain disease progression and heterogeneous treatment responses. We propose a \gls{BDRL} algorithm that optimizes agent-specific outcome distributions while enforcing comparability among similar agents and analyze its convergence. To further stabilize learning, we incorporate a post-update projection step formulated as a constrained convex optimization problem, which efficiently aligns individual outcomes with a high-performing reference within a specified tolerance. We apply our algorithm to manage hypertension in a large subset of the US adult population by categorizing individuals into cardiovascular disease risk groups. Our approach modifies treatment plans for median and vulnerable patients by mimicking the behavior of high-performing references in each risk group. Furthermore, we find that \gls{BDRL} improves the number and consistency of quality-adjusted life years compared with reinforcement learning baselines.
\end{abstract}



\glsresetall
\section{Introduction}\label{sec:Intro}
\Gls{RL} has achieved extraordinary success in complex decision-making domains like robotics, online advertising, and healthcare by optimizing how agents interact with unknown environments \citep{gaziStatisticalReinforcementLearning2026, baucum2022adapting}. This success largely relies on methods for learning policies that optimize the expected cumulative benefit from trial-and-error interactions. While expectation-focused objectives can be effective for optimizing average outcomes, they are often inadequate in highly uncertain environments or risk-sensitive settings. In such contexts, policies that are optimal on average may conceal significant variance, leading to ``high-risk, high-reward'' strategies that are unacceptable in risk-critical circumstances. Furthermore, accounting for expected performance and dispersion does not preclude substantial variation in realized benefits across a population of agents. This variation may arise among agents operating in the same environment and under comparable conditions. Learning algorithms may sacrifice the stability of individual agents to improve the mean of the aggregate. These challenges highlight the need for an approach that accounts for outcome distributions and explicitly addresses differences in agent-level performance.

\Gls{DRL} addresses the limitations of the classical expectation-based framework by modeling the full distribution of cumulative benefits  \citep{bellemare2017distributional}. By capturing outcome distributions, \gls{DRL} enables decision makers to assess the reliability and variability of an action rather than merely its average. However, standard \gls{DRL} methods do not explicitly regulate the consistency of outcome distributions across multiple comparable agents. Uncertainty from simulators or real-world interactions can amplify training instabilities and lead to significant divergence in the learned distributions of otherwise similar agents.

To bridge this gap, we propose a \gls{BDRL} algorithm that enforces the similarity of outcome distributions among comparable agents. Given a group of agents with similar characteristics, we regularize their outcome distributions using the 2-Wasserstein distance \citep{villani2009optimal}. This metric yields informative gradients even when supports are disjoint, stabilizing learning and encouraging similar outcome distributions across congruent agents. We also introduce a post-update projection step that constrains the learning dynamics by aligning individual estimates with a high-performing reference. This process modifies the behavior of low-performing agents by imitating high-performing agents in the same group, improving or ``boosting'' their outcome distributions.

\subsection{Healthcare Applications}
The demands of healthcare decision making motivate our \gls{BDRL} algorithm. In clinical settings, treatment policies are often deployed across heterogeneous patient populations that share similar diagnostics but exhibit highly variable responses and risks. Moreover, treatments with identical expected outcomes may possess vastly different risk profiles. For instance, a high-variance intervention may offer a chance of full recovery at the cost of significant risks of adverse events, whereas a low-variance alternative ensures stability. Optimizing for expected outcomes alone can obscure poor or unstable outcomes for particular patients.

A prominent example of this challenge arises in the management of chronic conditions, such as \gls{ASCVD}. This cardiovascular disease is a leading cause of death in the United States, with myocardial infarction and stroke as its predominant clinical manifestations \citep{kochanek2023national}. National statistics indicate that coronary heart disease, largely driven by myocardial infarction, accounts for 40.3\% of cardiovascular disease–related deaths, while stroke accounts for 17.5\% \citep{martin2024circulation}. Hypertension, also known as high \gls{BP}, is a leading modifiable risk factor for \gls{ASCVD} \citep{whelton2018guideline}. Treatment decisions in this setting are complicated by substantial heterogeneity in patient risk profiles and treatment responses \citep{sundstrom2023heterogeneity}. While clinical guidelines and classical \gls{RL} approaches can be valuable for hypertension treatment planning, they may yield policies with similar average outcomes yet markedly different variability in patient-level outcomes.

\subsection{Main Contributions}\label{sec: Contribution}
We aim to improve the consistency of \gls{DRL} in situations where the behavior and outcome information of one agent may help improve the performance of others. Our approach builds upon standard \gls{DRL}, constrained optimization, and the Wasserstein distance. The major contributions of this work are summarized as follows:

\begin{itemize}
    \item \textbf{We propose and analyze a new \gls{BDRL} algorithm.} The algorithm constraints standard \gls{DRL} and utilizes the 2-Wasserstein distance to regularize cumulative benefit (i.e., return) distributions. This approach ensures informative decision-making while promoting outcome consistency across groups of comparable agents. Importantly, \gls{BDRL} enhances the outcomes of low-performing agents without compromising the outcomes of the high-performing references. Our convergence analysis shows that our return updates yield geometric decay of alignment error and guarantee stable convergence of each agent’s learned return distribution.
    
    \item \textbf{We introduce a post-update projection step that significantly enhances training speed and stability in \gls{DRL}.} Our projection step converts an originally intractable optimization over probability measures into a convex problem, substantially accelerating training. Additionally, our projection method overcomes numerical instabilities, such as undefined values and vanishing gradients, commonly associated with \gls{KL} or \gls{JS} divergences in settings where the underlying distributions have disjoint supports.

    
    \item \textbf{We provide empirical evidence of the effectiveness of \gls{BDRL} via an \gls{ASCVD} case study.} Our test case compares \gls{BDRL} to previous \gls{RL} algorithms used in hypertension management in a large US population \citep{oh2022precision,ghasemi2025personalized}. We compare our approach with these baselines in terms of health outcomes and treatment consistency across patients.
\end{itemize}

\subsection{Article Organization}
The remainder of this paper is organized as follows. Section~\ref{sec:LitReview} reviews the literature related to our work. Section~\ref{sec:modeling} introduces our modeling formulation and notation. Section~\ref{sec:solution approach} presents the proposed \gls{BDRL} framework, including a 2-Wasserstein regularization and a post-update projection mechanism. Section~\ref{sec:Theory} provides theoretical results on convergence, stability, and computational efficiency. Section~\ref{sec:Results} reports numerical results for the hypertension treatment planning case study, along with sensitivity analyses and comparisons to \gls{RL} baselines. Finally, Section~\ref{sec:Discussion} discusses practical implications, limitations, and directions for future research.

\section{Literature Review}\label{sec:LitReview}

Our work draws from four closely related research streams: (i) \gls{RL} for clinical decision support; (ii) \gls{DRL}; (iii) constrained, safe, and fair \gls{RL}; (iv) metrics for distribution comparison. We summarize the most relevant contributions in each area and clarify how our approach differs.

\subsection{Reinforcement Learning for Clinical Decision Support}
The \gls{RL} paradigm has been widely studied as a framework for sequential decision-making, with canonical treatments in \citet{sutton2018reinforcement}. In healthcare, \gls{RL} has been applied to learn treatment strategies from observational cohorts and simulators. \citet{gottesman2019guidelines} provide practical guidelines for responsible \gls{RL} in healthcare, emphasizing evaluation pitfalls and the importance of robust off-policy reasoning. Surveys such as \cite{yu2023rlhealthcare} and \cite{banumathi2025reinforcement} further synthesize applications improving \gls{RL} training stability across chronic disease management, critical care, diagnosis, and operations.

In the realm of cardiovascular diseases, \citet{marrero2024model} propose a Q-learning framework for hypertension treatment planning that enforces a monotone policy structure, thereby improving interpretability for medical practitioners. \citet{drudi2024reinforcement} apply \gls{RL} to optimize treatment decisions for cardiovascular disease and report a 20\% reduction in heart failure mortality. \citet{zhou2025optimizing} use \gls{RL} to assign lipid-modifying therapies to patients with cardiovascular disease, improving heart-related outcomes by 37\%. \citet{ghasemi2025personalized} apply \gls{RL} to support treatment decisions for coronary artery disease and report a 32\% improvement in patient heart outcomes. \citet{zheng2021personalized} use \gls{RL} to guide prescription decisions for cardiovascular disease, achieving high concordance with clinicians’ prescriptions while substantially reducing cardiovascular risk outcomes. However, prior studies may fail to provide similar treatment for comparable patients, potentially leading to inconsistent outcomes. Our proposed \gls{BDRL} framework addresses this gap by explicitly enforcing consistent treatment across similar patients.

\subsection{Distributional Reinforcement Learning}
In contrast to traditional expectation-based learning methods, \gls{DRL} characterizes the complete return distribution. This characterization provides a more granular account of environmental stochasticity, specifically regarding risk-sensitive variability and tail-end phenomena. The foundational perspective in \citet{bellemare2017distributional} establishes the distributional Bellman framework and clarifies when distributional operators exhibit contraction properties. Subsequently, alternative distributional algorithms have been proposed, including quantile regression approaches \citep{dabney2018qrdqn}, implicit quantile networks \citep{dabney2018iqn}, and actor-critic variants \citep{barthmaron2018d4pg}. These methods demonstrate that distributional learning can improve policy quality and training dynamics, especially when the outcome variability is meaningful. Our work builds on this line of research by using distributional learning to support clinically motivated requirements. Rather than treating return distributions solely as a modeling enhancement, we directly regularize discrepancies across return distributions for clinically similar individuals, enabling explicit control over dispersion while still optimizing for clinical outcomes.

\subsection{Constrained, Safe, and Fair Reinforcement Learning}
A large body of literature studies safe and constrained \gls{RL}, including policy optimization under explicit constraints \citep{achiam2017cpo} and Lyapunov-based methods that guarantee constraint satisfaction during learning \citep{chow2018lyapunov}. Surveys such as \citet{garcia2015saferl} summarize broad approaches to safe \gls{RL} and their tradeoffs. Fairness in \gls{RL} has also received growing attention. \citet{jabbari2017fairrl} initiate a formal study of fairness in \gls{RL}, and recent surveys provide a broader taxonomy of definitions and methods \citep{reuel2024fairrlsurvey,chen2026survey}. Most constrained and safe approaches set boundaries on expected cumulative costs or risks, while fairness work often restricts expectation discrepancies across groups \citep{wen2021algorithms}.

Broadly, our framework represents a shift in how constraints are modeled in high-stakes sequential decision making. While standard constrained \gls{RL} typically treats safety as the mitigation of secondary, scalar cost metrics, our approach recognizes that in many critical settings, the primary concern is often the consistency of the outcomes themselves. Rather than imposing safety constraints on external costs, we advocate for comparability across return distributions for agents with similar baseline profiles. Our approach improves outcomes for low-performing agents by leveraging the strengths of high-performing references without degrading their results. In essence, our boosting objective is a special case of algorithmic fairness in which the outcomes of the best-performing agents cannot be adversely affected. By focusing on regularizing the return distribution of similar agents, the objective transitions from merely controlling expected performance to ensuring optimal distributional consistency in long-run outcomes.

\subsection{Metrics for Distribution Comparison}
A key design choice in distributional regularization is the discrepancy measure used to compare distributions. Many commonly used divergences are based on density ratios and can behave poorly when the distributions have limited or no overlap. For example, the \gls{KL} divergence can be undefined or infinite under disjoint supports, while the \gls{JS} divergence can saturate in this regime, producing weak learning signals \citep{arjovsky2017towards}. 

These limitations motivate the use of optimal transport, which provides geometry-aware distances between probability measures \citep{villani2009optimal, peyre2019computationalot}. In particular, Wasserstein distances remain well-defined even under weak overlap and can yield informative gradients that reflect the underlying geometry of the outcome space \citep{bellemare2017distributional}. As a result, \citep{arjovsky2017wgan, ghasemloo2025agglomerative} advocate using Wasserstein objectives to mitigate divergence-based failure modes, such as saturation and vanishing gradients, when the supports are nearly disjoint. Together, these considerations motivate our use of the 2-Wasserstein distance as a stable, geometry-aware discrepancy for distributional alignment.

\section{Problem Setting}\label{sec:modeling}
We adopt an infinite-horizon \gls{MDP} with finite state and action spaces as our mathematical framework \citep{puterman2014markov}. This section details our formulation and notation.

At each decision epoch, the system occupies a state $s\in\mathcal{S}$, and a decision maker selects an action $a\in\mathcal{A}$ according to a stationary policy $\pi:\mathcal{S}\mapsto\mathcal{A}$. After an action is taken, a next state $s'\in\mathcal{S}$ is generated from a stationary transition kernel $P(\cdot\mid s,a)$, and an immediate reward $r(s,a)$ is realized. The rewards $r(s,a)$ are bounded and independent random variables at each state $s$ and action $a$. The process then repeats indefinitely under a discount factor $\gamma\in(0,1)$ until hitting the absorbing states, which ensures the infinite accumulation of rewards is finite and well-defined.

Formally, our infinite-horizon \gls{MDP} is given by the tuple $(\mathcal{S}, \mathcal{A}, P, r, \gamma)$. For a stationary policy $\pi$ and an initial state $s$, we define the discounted return $Z^\pi(s,a)$ as a random variable satisfying the distributional Bellman recursion \citep{bellemare2017distributional}:
\begin{equation*}
    Z^\pi(s, a) \stackrel{D}{=} r(s, a) + \gamma Z^\pi(s', a'),
\end{equation*}
where $\stackrel{D}{=}$ denotes equality in distribution. An optimal policy $\pi^* \coloneqq (\pi^*(s):s \in \mathcal{S})$ is then constructed by selecting an action $\pi^*(s) \in \arg\max_{a \in \mathcal{A}} \mathbb E[Z^\pi(s, a)]$ at each state $s$.

\section{Boosted Distributional Reinforcement Learning}\label{sec:solution approach}
This section introduces our algorithm to enforce similar return distributions among a set of comparable agents, without compromising the outcomes of high-performing references. We consider settings with one central controller (e.g., a physician) optimizing policies for multiple, distinct subordinate agents (e.g., patients). Agents can be categorized into one or more groups with similar features, according to a prespecified, domain-informed definition (e.g., comparable \gls{ASCVD} risk). All agents share the same state and action spaces but have distinct reward and transition functions. 

Although we assume \gls{MDP} environments, they may not be known to the central controller. Consequently, our \gls{BDRL} algorithm is model-free, enabling it to derive optimal policies directly from interactions with the environment without explicitly estimating transition or reward functions.

\subsection{Algorithm}
Our \gls{BDRL} algorithm aims to learn the return distributions of $Z(s, a)$ for each state $s \in \mathcal{S}$ and action $a \in \mathcal{A}$ to identify optimal agent-level policies, while ensuring distributional consistency across similar agents. The process begins by partitioning an agent population $\mathcal{P} \coloneqq \{1,\ldots,N\}$ into $K$ groups based on their baseline features $X$ in Algorithm \ref{alg:cdrl-part-a}. Subsequently, the algorithm trains an initial \gls{DRL} procedure without return regularization to identify high-performing references in each group. The \gls{BDRL} procedure then penalizes discrepancies between the return distributions among two agents with the greatest distributional discrepancy in each of the groups using the 2-Wasserstein distance. To maintain training stability and satisfy a predefined return difference tolerance $\epsilon > 0$, we implement a post-update projection mechanism, which is detailed in Algorithm~\ref{alg:cdrl-part-b}.
Figure~\ref{fig:BDRL algo} summarizes the \gls{BDRL} algorithm.

\begin{figure}[h!]
    \centering
    \includegraphics[width=1\textwidth]{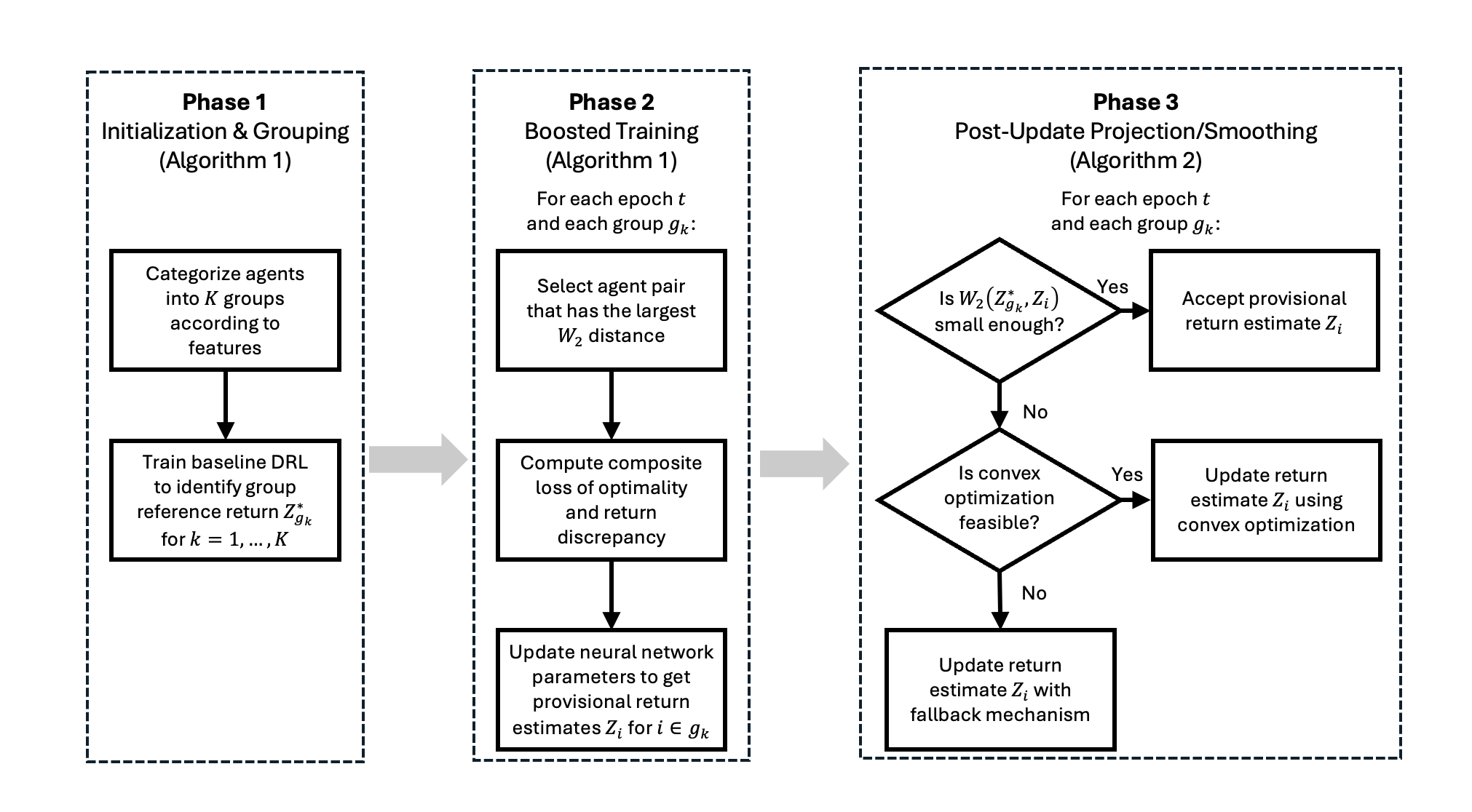} 
    \caption{Overview of the \gls{BDRL} algorithm.}
    \label{fig:BDRL algo}
\end{figure}

\subsubsection{Grouping and Training.}
Through the implementation of \gls{BDRL}, our goal is twofold: (i) to obtain improved returns for all agents and (ii) to ensure comparability of outcomes among agents who are similar in a domain-specific definition. As outlined in Algorithm~\ref{alg:cdrl-part-a}, we operationalize comparability via \emph{within-group boosting}, where agents are first partitioned into homogeneous groups $g_k \in G$ based on their baseline features $X$. For example, this feature vector may include critical indicators such as age and systolic \gls{BP} for the management of \gls{ASCVD}. 

This partitioning step is a deliberate safeguard: by clustering agents with similar profiles, we ensure that the 2-Wasserstein distance regularization is applied only to comparable agents. When agents are substantially different, it may not be meaningful to enforce comparable return distributions. For example, enforcing comparability between a 45-year-old with mild hypertension to an 80-year-old with multiple comorbidities may not be clinically or ethically desirable. Grouping agents by their baseline features $X$ allows \gls{BDRL} to optimize for high-performing group references that are realistic and achievable for every member within group $g_k$.

After grouping, we train \gls{BDRL} using data generated from agent $i$'s trajectories $\bm{\omega}_i = (\omega_i^1, \omega_i^2, ...,\omega_i^L)$. Each trajectory $\omega_i^l$ with $l = 1,2,..., L$ contains transitions $\omega_i^l = (s_0, a_0, r_0, s_{1}, a_{1}, r_{1},...)$, where actions $a_t$ follow an $\epsilon$-greedy policy \citep{sutton2018reinforcement}. As in standard \gls{DRL} \citep{bellemare2017distributional}, we initialize a neural network parameter $\theta$, which is used to approximate the return distributions $Z(s,a)$. This parameter is iteratively updated via the \gls{BDRL} training. For each agent $i \in \mathcal{P}$, we first train an optimality-only \gls{DRL} to get an unregularized return estimate $Z_i^{\mathrm{opt}}$. Then, for each group $g_k$, we select a high-performing in-group reference by identifying the agent with the largest expected return, i.e., $i_{g_k}^\star \in \arg\max_{i\in g_k}\mathbb{E}[Z_i^{\mathrm{opt}}]$. We then set $Z_{g_k}^{\!*}\gets Z_{i_{g_k}^\star}^{\mathrm{opt}}$ as the reference for subsequent boosting in group $g_k$. 

For each group, we sample a minibatch $\mathcal{B}$ of state-action pairs $(s_t,a_t)$ from its trajectories. Afterwards, we construct a loss function $ \mathcal{L}(\theta)$ to update the neural network parameter $\theta$. The loss contains two components: (a) an accuracy term that measures the divergence between the predicted return distribution and its distributional Bellman target, and (b) a constraint violation term that penalizes differences between the return distributions of two agents \(i,j\in\mathcal{P}\) drawn from the same group $k$ with the largest distribution discrepancy according to the 2-Wasserstein distance. The accuracy term is measured by the loss: 
\begin{align*}
    -\sum_{(s,a)\in\mathcal{B}}\sum_{d=1}^D\Large(&m^i_d(s,a)\log p^i_d(s,a)\\
    &+m^j_d(s,a)\log p^j_d(s,a)\Large),
\end{align*}
where $p^i_d(s,a)\in[0,1]$ is the probability mass assigned by agent $i$ to a return distribution support point $z_d$, and $m^i_d(s,a)\in[0,1]$ is the corresponding one-step distributional Bellman target. We use the return atoms $\{z_d\}_{d=1}^{k}$ to represent the support of a return distribution with a constant spacing denoted by $\Delta z \in \mathbb{R}_{>0} $. For a sampled transition $(s,a,r,s')$, the target distribution is induced by $r+\gamma Z_i(s',a')$, where $a'=\pi(s')$. Because $r+\gamma Z_i(s',a')$ generally does not lie on the fixed support $\{z_d\}_{d=1}^{k}$, we obtain $m^i_d(s,a)$ by projecting this shifted-and-discounted next-state distribution back onto the atoms via the standard categorical projection in \cite{bellemare2017distributional}. The constraint violation term is measured by the following loss: 
\begin{align*}
    \Wtwo(Z_i,Z_j) \coloneqq \Large(&\sum_{d=1}^D [F_i(z_d\mid s,a)\\
    &-F_j(z_d\mid s,a)]^2\,\Delta z\Large)^{1/2},
\end{align*}
where $\Wtwo$ denotes the 2-Wasserstein metric used to quantify the discrepancy between the cumulative distribution functions $F_i$ and $F_j$ of random variables $Z_i$ and $Z_j$, respectively. Unlike \gls{KL} and \gls{JS} divergence, the Wasserstein metric avoids the numerical instability associated with logarithmic probability ratios. This characteristic ensures more stable gradients and superior training convergence, as we will demonstrate in Section \ref{sec:Results}. We use the Lagrangian multiplier $\lambda \in \mathbb{R}_{\geq0}$ as a penalty that balances the trade-off between accuracy and within-group return comparability. Next, we update $\theta$ using gradient descent \citep{tan2021reinforcement} and predict an updated provisional estimate $Z(s,a)$ until convergence. We obtain the target policy by selecting the action $\pi^*(s) = \arg\max_a\mathbb E[Z(s, a)]\ \forall s$.

\begin{algorithm}[t]
\caption{Boosted DRL: Grouping and Training}
\label{alg:cdrl-part-a}
\small
\vskip4pt
\begin{algorithmic}[1]
\Require Training data $\mathcal{D}=\{(X_n,\omega_n)\}_{n=1}^N$; groups $K$; return atoms $\{z_d\}_{d=1}^{D}$ with spacing $\Delta z$; weight $\lambda>0$; tolerance $\epsilon>0$; mix parameters $\alpha\in(0,1)$ (solved), $\rho\in(0,1)$ (fallback); epochs $T$
\Ensure Policy $\pi$; per-patient return distributions $\{Z_n(s,a)\}, \forall (s, a)$

\Procedure{Boosted}{$\mathcal{D},K,\{z_d\},\lambda,\epsilon,\rho,T$}
  \State \textbf{Group:} partition patients by $X$ into $G=\{g_1,\dots,g_K\}$
  \State Initialize parameters $\theta$ of a distributional \gls{RL} model (\gls{DRL} head)
  \State \textbf{Baseline optimality:} for each $n$, train an optimality-only head to get $Z_n^{\mathrm{opt}}(s,a)$
  \For{$g_k=1$ \textbf{to} $K$}
      \State $n_{g_k}^\star \gets \arg\max_{n\in g_k}\ \mathbb{E}[Z_n^{\mathrm{opt}}]$
      \State $Z_{g_k}^{\!*}(s,a) \gets Z_{n_{g_k}^\star}(s,a)$ \Comment{group reference}
  \EndFor

  \For{$t=1$ \textbf{to} $T$} \Comment{main training with boosting}
    \For{$g_k=1$ \textbf{to} $K$}
      \State Sample minibatch $\mathcal{B}=\{(s,a)\}$ from $g_k$
      \State \textbf{Pair selection:} 
             $(i,j)\in\arg\max_{n_1\neq n_2\in g_k}\ \sum_{(s,a)\in\mathcal{B}}\Wtwo\big(Z_{n_1}(s,a),Z_{n_2}(s,a)\big)$
      \State For each $(s,a)\in\mathcal{B}$, compute $p^i_d(s,a),p^j_d(s,a)$ over atoms $\{z_d\}_{d=1}^D$;
             let $m^i_d(s,a),m^j_d(s,a)$ be target distributions (Bellman projections)
      \State \textbf{Constraint violation:} 
             $\Wtwo(Z_i,Z_j)=\big(\sum_{d=1}^D [F_i(z_d\mid s,a)-F_j(z_d\mid s,a)]^2\,\Delta z\big)^{1/2}$
      \State \textbf{Objective:}
             $\mathcal{L}(\theta)=\sum_{(s,a)\in\mathcal{B}}\!\Big[-\sum_{d=1}^D(m^i_d\log p^i_d+m^j_d\log p^j_d)+\lambda\,\Wtwo(Z_i,Z_j)\Big]$
      \State Update $\theta$ with gradient descent
      \State \textsc{PostUpdate}$(g_k,\mathcal{B},i,j)$ \Comment{Algorithm \ref{alg:cdrl-part-b}}
    \EndFor
  \EndFor

  \State \textbf{Deployment (reference agents):} For each reference agent, use
  \State \hspace{1.5em} $\pi_{g_k}^{*}(s)\in\arg\max_a \mathbb{E}[Z_{g_k}^{\!*}(s,a)]$
  \State \textbf{return} $\pi^*(s) = \arg\max_a\mathbb E[Z(s, a)]\ \forall s$, and $\{Z_n(s,a)\}_{n=1}^N, \forall(s,a)$
\EndProcedure
\end{algorithmic}
\end{algorithm}

\subsubsection{Post-Update Projection.}
Comparability should not come at the expense of performance. Algorithm~\ref{alg:cdrl-part-b} introduces a post-update projection/smoothing step to align each individual’s estimate with a high-performing in-group reference. 

\begin{algorithm}[t]
\caption{Boosted DRL: Post-Update Projection / Smoothing}
\label{alg:cdrl-part-b}
\small
\vskip4pt
\begin{algorithmic}[1]
\Require Group $g_k$; batch $\mathcal{B}$; updated agents $i \in \mathcal{P}$; refs $Z_{g_k}^{\!*}(s,a)$; tolerance $\epsilon$; fallback $\rho$; solver for $\alpha\in(0,1)$; 
\Ensure Smoothed returns $Z_i^{\text{new}}(s,a)$ for each agent $i \in \mathcal{P}$

\Procedure{PostUpdate}{$g_k,\mathcal{B},i$}
      \For{$i \in g_k$}
        \State Obtain updated $Z'_i(s,a)$ from the network
        \For{each $(s,a)\in\mathcal{B}$}
          \If{$\Wtwo(Z_{g_k}^{\!*}(s,a),Z'_i(s,a))<\epsilon$} \Comment{distributional distance constraint met}
             \State \textbf{Accept} $Z_i'(s,a)$
          \ElsIf{Problem \eqref{eq:convex_opt} is feasible} \Comment{convex optimization}
             \State Solve for $\alpha\in(0,1)$:
                   $\min_{\alpha\in(0,1)} \Wtwo(\alpha Z_i+(1-\alpha)Z'_i,\ Z'_i)$
                   s.t. $\Wtwo(\alpha Z_i+(1-\alpha)Z'_i,\ Z_{g_k}^{\!*})\le\epsilon$
             \State $Z_i'(s,a)\gets \alpha Z_i(s,a)+(1-\alpha)Z'_i(s,a)$
          \Else \Comment{fallback mechanism}
             \State $Z_i'(s,a)\gets \rho\,Z_i(s,a)+(1-\rho)\,Z_{g_k}^{\!*}(s,a)$
          \EndIf
        \EndFor
      \EndFor
\EndProcedure
\end{algorithmic}
\end{algorithm}

Given a network’s provisional estimate \(Z'_{i}\), we compare it to the group reference \(Z^{*}_{g_k}\). If \(\Wtwo\!\big(Z^{*}_{g_k}, Z'_{i}\big) < \epsilon\), the estimate is accepted. If the constraint is not met, we form a corrected estimate by interpolating between the agent's pre-update distribution $Z_i$ (carried over from the last time we saw the same state-action pair) and the post-update provisional distribution $Z'_i$ produced by the neural network. Interpolating keeps the update conservative: it nudges an infeasible provisional estimate $Z'_i$ back toward feasibility while preserving information from the latest gradient step. We seek an $\alpha \in (0, 1)$ by solving the following optimization problem:
\begin{equation}\label{eq:convex_opt}
\begin{aligned}
    \underset{\alpha \in (0,1)}{\text{min}} \quad & W_2(\alpha Z_i + (1-\alpha)Z'_i, Z'_i) \\
    \text{s.t.} \quad & W_2(\alpha Z_i + (1-\alpha)Z'_i, Z_{g_k}^{\!*}) \le \epsilon,
\end{aligned}
\end{equation}
where $\epsilon>0$ is the maximum allowable distributional distance between an agent's return and the group reference. The objective of this optimization is to find a combined estimate that satisfies the return parity constraint while remaining as close as possible to the new estimate $Z'_i$. This balance ensures the learned estimate is not distorted excessively but still satisfies the constraints. If a solution $\alpha$ is found, we update the estimate as $Z_i' \leftarrow \alpha Z_i + (1-\alpha)Z'_i$. However, if no solution $\alpha$ exists for this problem, we resort to a fallback mechanism. We update the estimate as a linear combination of the previous estimate $Z_i$ and the optimal group estimate $Z_{g_k}^*$. That is, we use \(Z_{i}'\gets \rho Z_{i}+(1-\rho)Z^{*}_{g_k}\) with \(\rho\in(0,1)\). Overall, the post-update projection maintains estimation accuracy while reducing within-group dispersion, thereby iteratively shrinking the gap between an agent’s current estimate and the group-optimal reference as the iterative index $t$ becomes larger. We measure the gap between these return estimates through $d^{(t)} \coloneqq W_2(Z_i^{(t)}, Z_{g_k}^{*})$, where $Z_i^{(t)}$ denotes the return estimate at iteration $t$.

\section{Theoretical Analysis and Properties}\label{sec:Theory}

We now provide a comprehensive analysis of our proposed algorithm, focusing on two critical dimensions: convergence stability and computational efficiency. The
proofs of our claims can be found in Appendix \ref{app:theory_results}. 


    
    

    


\subsection{Convergence and Stability Analysis}
In \gls{BDRL}, ensuring that the iterative updates of distributional estimates do not diverge is paramount. We define the alignment error $d^{(t)}$ as the 2-Wasserstein distance between these distributions denoted by $d^{(t)} \coloneqq W_2(Z_i^{(t)}, Z_{g_k}^{*})$. The $d^{(t)}$ metric serves as a scalar proxy for the estimate's optimality, representing the minimum ``transport cost" required to transform the current estimated outcome distribution of agent $i$ into the optimal group distribution of $Z_{g_k}^{*}$.

We now establish the stability of our update rule. A fundamental property of the Wasserstein metric is its behavior under mixture distributions $\alpha Z_i + (1-\alpha)Z'_i$ (or $\rho Z_i+(1-\rho)Z^{*}_{g_k}$). The following theorem establishes that mixing the current estimate with a target distribution effectively contracts this error space. 

\begin{theorem}[The Contraction of Mixtures]
\label{thm:Contraction of Mix}
The distance between the updated distributional estimate $Z_i^{(t)}$ and the optimal group estimate $Z_{g_k}^{*}$ is nonexpansive under the mixture operation. Moreover, the distance at the subsequent step satisfies a $d^{(t+1)} \leq \max \{d^{(t)}, \epsilon \}$ update, where $\epsilon>0$ is a predefined return difference tolerance.
\end{theorem}
This contraction property is the driving force behind the algorithm's convergence. Because the mixture operation strictly reduces the Wasserstein distance by a factor $\rho < 1$ (whenever the error exceeds $\epsilon$), the sequence of errors is bounded by a geometric progression: $d^{(t)} \leq \rho^t d^{(0)}$. This geometric decay guarantees that the distributional estimate does not merely oscillate but is systematically pulled towards the optimal distribution, ensuring that it enters the $\epsilon$-neighborhood in a finite number of steps $N \ge \lceil \ln(\epsilon/d^{(0)})^2/\ln\rho \rceil$, where $ \lceil x \rceil \coloneqq \min\{y \in \mathbb{Z} | y \ge x\} $.

\paragraph{Proof Sketch.} The proof of Theorem \ref{thm:Contraction of Mix} relies on the convexity of the squared Wasserstein distance. By constructing a mixture distribution $\alpha Z_i + (1-\alpha)Z_i'$, we utilize the linearity of marginals to show that the cost of the optimal transport plan for the mixture is bounded by the convex combination of the costs of the individual components. This geometric property ensures that the mixture never lies ``outside" the path connecting the current estimate and the target. Theorem \ref{thm:Contraction of Mix} implies that if a perfect update is impossible, a partial update will reduce the geometric distance to the target $Z^{*}_{g_k}$, providing a ``safety mechanism" often missing in standard distributional Bellman updates.

\paragraph{Healthcare Implications.} As an example, consider a clinician adjusting a medication dosage. They rarely switch from their current protocol to a radically different one instantly \citep{caffrey2020art}. Instead, they may ``mix" new evidence into their current protocol. This theorem guarantees that such intermediate steps mathematically move a patient's outcome distribution closer to the optimal recovery profile, preventing erratic treatment swings.

\begin{proposition}[$W_2$ Distance Convergence]
\label{thm:Nonexpansive}
The update rule guarantees that the $W_2$ distance between the current estimate and the optimal estimate remains nonexpansive. Consequently, the eventual $W_2$ distance will stay within the $\epsilon$ range.
\end{proposition}

While Theorem~\ref{thm:Nonexpansive} ensures local stability, our ultimate goal is long-term convergence. By recursively applying the nonexpansiveness property, we can derive a global convergence guarantee. The result in Proposition \ref{thm:Nonexpansive} confirms that each agent's learned distribution asymptotically approaches the performance of the best-performing reference in its group as training progresses, up to the tolerance level $\epsilon$. Furthermore, our result addresses the ``catastrophic forgetting'' problem often encountered in \gls{DRL} \citep{schwarz2018progress}, where ongoing updates under a changing data distribution can overwrite previously learned value/distributional estimates and degrade performance on previously mastered states. Our guarantee shows that enforcing distributional parity constraints does not destabilize the optimal policy, thereby reducing the risk of such forgetting.

\paragraph{Proof Sketch.} The proof of this claim examines two cases. If an unconstrained update $Z'$ is within the valid $\epsilon$-ball region, stability is trivial. If $Z'$ violates the constraint, we project it back via the mixture operation defined in Theorem~\ref{thm:Contraction of Mix}. Since the mixture path is convex and contractive, the projected point $Z^{(t+1)}$ is closer to the optimum than the previous point $Z^{(t)}$.

\paragraph{Healthcare Implications.} In critical care, ``do no harm" is the ultimate priority. Proposition \ref{thm:Nonexpansive} provides a theoretical guarantee that the algorithm's update to a treatment policy will never result in a distribution of outcomes that is worse than the current policy, effectively serving as a digital guardrail against unsafe exploration.




\subsection{Computational Efficiency via Convex Reformulation}
While the theoretical convergence properties are robust, a significant practical challenge in Algorithm~\ref{alg:cdrl-part-b} is the solution of the optimization problem involving Wasserstein constraints. Direct optimization over probability measures is often computationally prohibitive \citep{peyre2019computational}. To address this challenge, we demonstrate that the problem can be reduced to a standard form.

\begin{theorem}[Convex Quadratic Transformation]
\label{thm:Convex Opti}
The conditional optimization problem:
\begin{align*}
& \min_{\alpha} \ W_2(\alpha Z_i(s,a) + (1-\alpha)Z_i'(s,a),Z_i'(s,a)) \\
& \text{s.t.} \ \ W_2(\alpha Z_i(s,a) + (1-\alpha)Z_i'(s,a),Z_i^{*}(s,a)) \leq \epsilon \\
& \ \ \ \ \ \ \alpha \in (0,1)
\end{align*}
can be reformulated as a convex quadratic optimization problem.
\end{theorem}

The transformation in Theorem \ref{thm:Convex Opti} enables efficient numerical solutions with standard quadratic programming solvers, substantially improving the computational efficiency of the training loop. For example, the optimization problem can be solved efficiently using Gurobi Optimizer software by leveraging the computational method proposed by \cite{li2025efficient}.

Traditional distributional methods often rely on projection steps that are computationally expensive or use approximation heuristics that lack strict constraint guarantees (e.g., quantile regression losses). By reducing our projection to a single dimension quadratic program, we achieve exact constraint satisfaction with negligible computational overhead. Establishing the convex quadratic structure is nontrivial because it requires carefully characterizing how the $W_2$ distance behaves under linear mixtures of random variables while holding the supports fixed. In particular, its proof must rigorously justify that the Wasserstein metric remains well-defined and differentiable with respect to the mixing parameter $\alpha$. In addition, no pathological cases should arise when distributions have non-overlapping supports or degenerate mass points.

\paragraph{Proof Sketch.} Under fixed distributional support, the squared Wasserstein distance $W_2^2$ between two distributions becomes a weighted sum of squared differences of their cumulative distribution functions. Since the mixture distribution is linear in $\alpha$, the squared difference becomes quadratic in $\alpha$. The objective and constraints thus reduce to the form $A\alpha^2 + B\alpha + C \leq D$ for $A >0$ and $B, C, D \in \mathbb{R}$, which is a standard quadratic program.

\paragraph{Healthcare Implications.} Theorem \ref{thm:Convex Opti} considerably improves the efficiency of the “safety check” in \gls{BDRL}. These computational gains enable us to apply our algorithm to large collections of agents, such as the cohort in our case study, representing more than 16 million patients.

\section{Case Study}\label{sec:Results}
This section presents the empirical evaluation of our proposed algorithm, considering a physician who acts as a central controller optimizing policies for distinct patients as ``subordinate agents." We begin by detailing the \gls{MDP} parameterization and data sources that underpin our case study. Subsequently, we analyze the performance of the \gls{BDRL} algorithm, demonstrating its ability to optimize patient outcomes while satisfying within-group return difference constraints. 

It is important to emphasize that the \gls{BDRL} algorithm operates in a completely model-free manner. We construct \glspl{MDP} solely to serve as generative models for simulating patient trajectories. Our algorithm learns the optimal policy directly from the generated trajectories, without explicitly learning the underlying transition functions or reward functions. Furthermore, to capture realistic population heterogeneity, we model each patient as a distinct \gls{MDP}. While the state and action spaces are shared across the population, the specific transition probabilities and reward structures remain unique to each individual.

\subsection{Markov Decision Process Formulation Parameters}
Adapting the framework established by \cite{schell2016data}, we formulate the hypertension treatment planning problem as an infinite-horizon \gls{MDP}. Although major clinical guidelines for hypertension often emphasize a 10-year risk horizon for evaluation and communication of cardiovascular risk \citep{whelton2018guideline}, this window is primarily intended for risk assessment rather than for prescribing the full lifetime sequence of antihypertensive medications. In contrast, chronic hypertension management is an ongoing process that typically extends over the patient’s lifetime, motivating an infinite-horizon formulation \citep{zhou2024policy}. Our objective is to identify a stationary treatment policy that maximizes patients' total discounted \glspl{qaly}. Motivated by prior work in this area \citep{garcia2024interpretable, oh2022precision}, we consider discrete yearly decision epochs $t = 0,1,2,\dots$, terminating only upon transition to an absorbing death state $s_{\text{death}}$. Table~\ref{tab:base-case-params} lists our parameters and their sources. The components of the \glspl{MDP} $(\mathcal{S}, \mathcal{A}, f, r, \gamma)$ used in our cases are defined as follows:

\begin{itemize}
    \item \textbf{State Space ($\mathcal{S}$):} The state space encodes each patient’s \gls{ASCVD} risk profile, including demographics (i.e., age, sex, race, smoking status), clinical measurements (i.e., \gls{BP} readings, diabetes status, cholesterol), and overall health condition. Patients' overall health condition accounts for their history of \gls{ASCVD} events and it is classified into following ten mutually-exclusive categories: healthy (i.e., no history of \gls{ASCVD}), history of myocardial infarction but no adverse event in the current year, history of stroke but no adverse event in the current year, history of myocardial infarction and stroke but no adverse event in the current year, survival of a myocardial infarction, survival of a stroke, death from a non-\gls{ASCVD} related cause, death from a myocardial infarction, death from stroke, and dead ($s_{\text{death}}$).

    \item \textbf{Action Space ($\mathcal{A}$):} Following the established methodology of \citet{sussman2013using}, we abstract the action space to represent the \textit{intensity} of treatment rather than specific pharmacological classes. The action space consists of combinations ranging from 0 to 5 antihypertensive medications at half or standard doses, yielding $|\mathcal{A}| = 21$ distinct treatment choices. This abstraction is grounded in clinical evidence suggesting that the cumulative number of standard doses is the primary determinant of blood pressure reduction \citep{law2003value,blood2014blood}. By modeling treatment intensity, we maintain computational tractability while capturing the essential decision of therapy intensification.
    
    \item \textbf{Transition Kernel ($P$):} The stationary transition kernel $P$ models the progression of patient health. It is derived from risk models for \gls{ASCVD} events \citep{yadlowsky2018clinical}, treatment efficacy \citep{blood2014blood}, and mortality rates \citep{national2017health, kochanek2023national}. Although patient risk profiles naturally evolve over time, we model $P$ as stationary by incorporating age explicitly within the state space $s$. This idea ensures that the system dynamics remain consistent, while the resulting state transition update dynamically as the patient ages within the state space. Consistent with prior work \citep{garcia2024interpretable,schell2016data, marrero2024model}, we assume independence between myocardial infarction and stroke events, weighting their risks at 70\% and 30\% respectively \citep{tsao2023heart}. In addition, we assume that patients with a prior history of \gls{ASCVD} events are more likely to experience subsequent heart attacks or strokes. We reflect this assumption by increasing patients' heart attack and stroke odds when they have a history of either \gls{ASCVD} event \citep{bronnum2001survival,burn1994long}.
    
    \item \textbf{Reward Function ($r$):} The reward is defined as the quality-of-life weight associated with the patient’s health condition \citep{kohli2019cost}, minus the medication disutility for treatment at half or full dose \citep{sussman2013using}.

    \item \textbf{Discount Factor ($\gamma$):} To reflect the preference for immediate over delayed health benefits, future rewards are discounted by a factor of $\gamma = 0.97$ \citep{neumann2016cost}.
\end{itemize}

\begin{table}[t]
\centering
\small 
\setlength{\tabcolsep}{4pt} 
\caption{Base case parameters.}
\label{tab:base-case-params}
\begin{tabularx}{\linewidth}{l c Y}
\toprule
Parameter & Value & Source \\
\hline
BP reduction: standard dose (half dose) & & \\
\hspace{3mm} Systolic BP & 5.5 (3.7) mm Hg & \citep{blood2014blood}, \citep{sussman2013using}\\
\hspace{3mm} Diastolic BP & 3.3 (2.2) mm Hg & \citep{blood2014blood}, \citep{sussman2013using}\\
Risk for ASCVD events & Varies by patient & \citep{yadlowsky2018clinical} \\
ASCVD risk reduction: standard dose (half dose) & & \\
\hspace{3mm} Myocardial infarction & 13\% (7\%) & \citep{blood2014blood}, \citep{sussman2013using} \\
\hspace{3mm} Stroke & 21\% (14\%) & \citep{blood2014blood}, \citep{sussman2013using} \\
ASCVD risk due to myocardial infarction & 70\% & \citep{tsao2023heart} \\
Mortality from ASCVD events & & \\
\hspace{3mm} Myocardial infarction & Varies by patient & \citep{national2017health} \\
\hspace{3mm} Stroke & Varies by patient &  \citep{national2017health} \\
Treatment-related disutility & & \\
\hspace{3mm} Half dose & 0.001 & \citep{schell2016data}, \citep{sussman2013using} \\
\hspace{3mm} Full dose & 0.002 & \citep{schell2016data}, \citep{sussman2013using} \\
Life expectancy & Varies by patient & \citep{kochanek2023national} \\
Non-ASCVD mortality & Varies by patient & \citep{kochanek2023national} \\
\hline
\end{tabularx}
\end{table}

\subsection{Data Source and Parameterization}
To parameterize our generative \glspl{MDP}, we utilize data from the National Health and Nutrition Examination Survey spanning 2009 to 2016 \citep{stierman2021national}. Our primary analysis cohort consists of Black and White individuals aged 50 to 54 years without a prior history of stroke or myocardial infarction, representing a weighted population of approximately 16.72 million. We select this specific age demographic because it represents a relatively young population with a high prevalence of \gls{ASCVD} risk factors, offering a significant window for preventative intervention via hypertension treatment \citep{tsao2023heart, whelton2018guideline}.

Missing data are imputed using the MissForest package in R. To model the longitudinal evolution of patient risk factors, we estimate progression trajectories using linear regression. We regress untreated systolic blood pressure, total cholesterol, high-density lipoprotein, and low-density lipoprotein against demographic and health indicators, including age, sex, race, smoking status, and diabetes status. To ensure patient-specific accuracy, the intercept term of each regression model is calibrated by adding the residual difference between the observed clinical values and the fitted estimates. This health progression modeling allows us to reduce our state space to the ten mutually-exclusive health conditions \citep{garcia2024interpretable,schell2016data, marrero2024model}.

\subsection{Simulation Framework}
Because our constraints are most informative when comparing patients with similar clinical profiles, we partition patients into groups based on baseline health conditions meaningful for selecting treatment plans. Following the major hypertension treatment guidelines \citep{whelton2018guideline}, we apply $k$-means clustering to patients' demographics and clinical measurements considered risk factors for \gls{ASCVD} \citep{niedermayer2024clusters}. We identify $k=3$ as the optimal number of clusters for our population (Appendix \ref{app:clustering}). The resulting groups correspond to distinct profiles, which we interpret as patients with \emph{low risk}, \emph{intermediate risk}, and \emph{high risk} for \gls{ASCVD} events. 

We construct a simulation framework to evaluate our \gls{BDRL} algorithm. Before computing treatment plans, we evaluate each patient’s annual risk of ASCVD events. These risk estimates are then used to infer transition kernels and specify the corresponding dynamics in the \glspl{MDP}. Afterwards, based on the estimated transition and outcome models, we construct a stochastic behavioral policy using an $\epsilon$-greedy rule over the model-implied action value function $Q(s,a)=\mathbb{E}[Z(s, a)]$ \citep{sprouts2022development}: with probability $1-\epsilon$ it selects the highest-valued feasible action, and with probability $\epsilon$ it selects a feasible action uniformly at random. Trajectories generated under this policy are subsequently used as input data for \gls{BDRL}, which learns an estimate of the return distribution. Based on previous literature \citep{pham2017predicting}, we use $L=42$ trajectories for all agents. 

To ensure the statistical reliability of our simulation outputs, we first determine the appropriate minibatch size for partitioning the simulation trajectories. Typically, the batch size ranges from 64 to 4,096 \citep{schulman2017proximal}. The details of our batch size determination process are provided in Appendix \ref{app:convergence_diagnostics}. Based on this batch specification, we first train a base \gls{BDRL} algorithm with $\lambda=0.1$ using the simulated trajectory data. Figure~\ref{fig:simulation_framework} shows our simulation framework.

\begin{figure}[h!]
    \centering
    \includegraphics[width=1\textwidth]{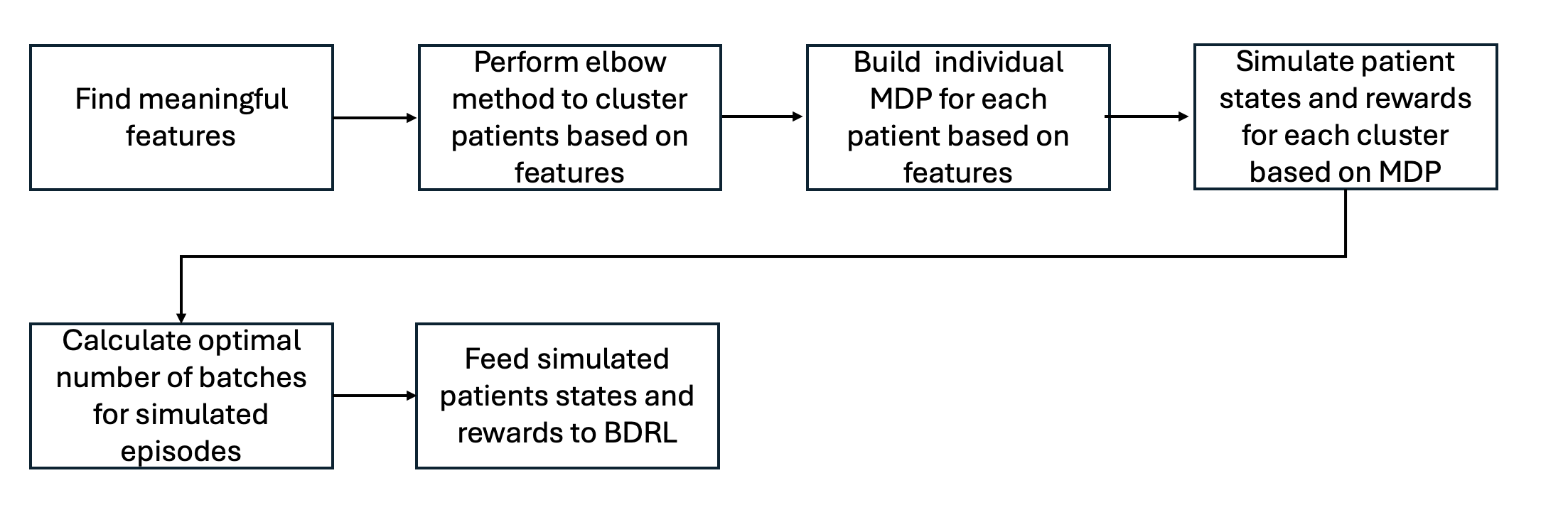} 
    \caption{Overview of the simulation framework.}
    \label{fig:simulation_framework}
\end{figure}

\subsubsection{Analysis.} Our analysis begins by quantifying the shift before and after \gls{BDRL} training, evaluating whether the policy behavior among agents within the same groups becomes closer. We also evaluate the convergence of return distributions among patients within the same risk group by tracking the trajectory of the $W_2$ distance between its two most different patients throughout the learning process in Appendix \ref{app:convergence_return_dist}. Next, we compare \gls{BDRL} against baseline methods, including standard \gls{DRL}, Deep Q-Network, and Q-learning \citep{oh2022precision, ghasemi2025personalized}. We subsequently train \gls{BDRL} over a range of penalty weights $\lambda$ to evaluate their effect on the outcomes of vulnerable patients. Next, we perform sensitivity analyses with respect to the return-difference tolerance $\epsilon$ and the fallback weight $\rho$ to study the impact of hyperparameter choices. Finally, we explore the consequences of using \gls{KL} or \gls{JS} divergences to regularize return distribution differences. 











\subsection{Numerical Results}
In this section, we evaluate the clinical and population-level implications of \gls{BDRL}. Consistent with the chronic nature of hypertension management, we present outcomes for patients in the initial low-risk, intermediate-risk, and high-risk (e.g., following a stroke) subgroup defined by $k$-means clustering, to demonstrate the versatility of the policy across different stages of disease progression.

\subsubsection{Insights from Return Distribution Boosting.}
We begin by comparing the unconstrained \gls{DRL} against our proposed \gls{BDRL}. This comparison isolates the specific clinical gains attributable to the $W_2$ regularization and the boosting mechanism. 

\paragraph{Policy Behavior.} Figure~\ref{fig:treatment_distn} illustrates the distribution of the prescribed treatment intensities before and after boosting for the most vulnerable patient (i.e., the lowest-performing agent) compared to the most resilient patient (i.e., the highest-performing agent) across our risk clusters. The boosted policy reallocates treatment intensity to reduce outcome variance: it shifts towards milder interventions for the low-risk group, moderate interventions for the intermediate-risk group, and slightly more aggressive interventions for the high-risk group. 

\begin{figure}[h!]
    \centering
    \includegraphics[width=0.9\textwidth]{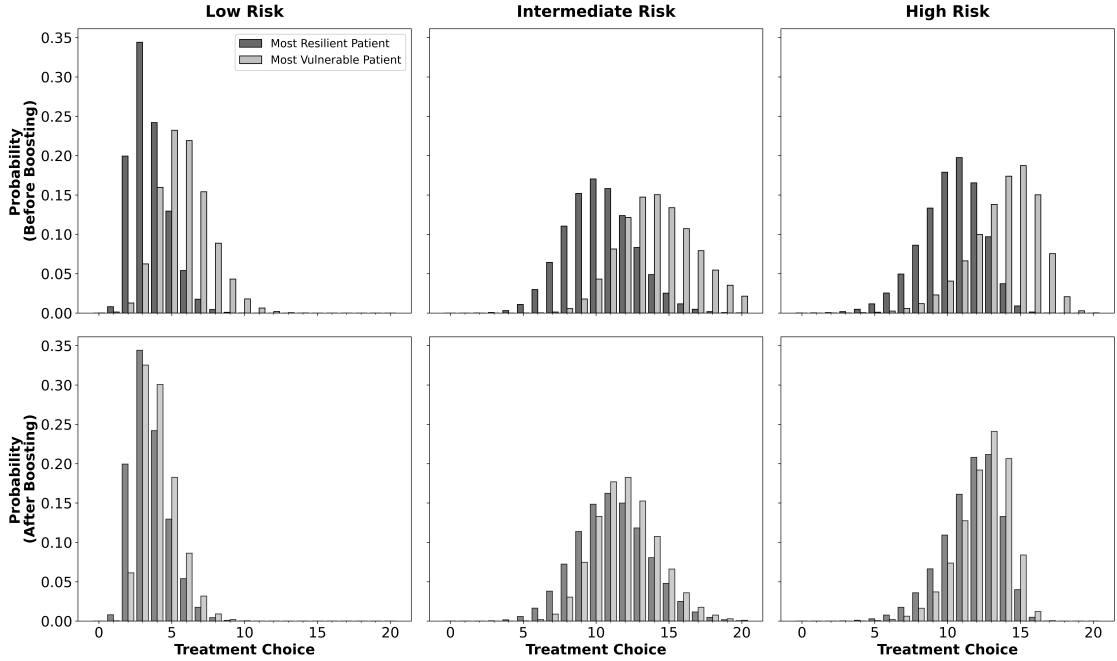} 
    \caption{Probability of action selection during training for the most vulnerable patient and the most resilient patient across risk clusters before and after boosting. The treatment choice ranging from 0 to 20 represents the index of each of the 21 actions considered in the \glspl{MDP}.}
    \label{fig:treatment_distn}
\end{figure}

\paragraph{Health Outcomes.} The similarity of treatment decisions translates into consistent return distributions, ensuring comparable health outcomes for patients within the same group. The resulting \glspl{qaly} by risk group are shown in Figure \ref{fig:Learned Distribution Plots}. The plot shows that \gls{BDRL} brings return distributions within each group closer together, promoting congruent outcomes across patients.

\begin{figure}[h!]
    \centering
    \includegraphics[width=0.9\linewidth]{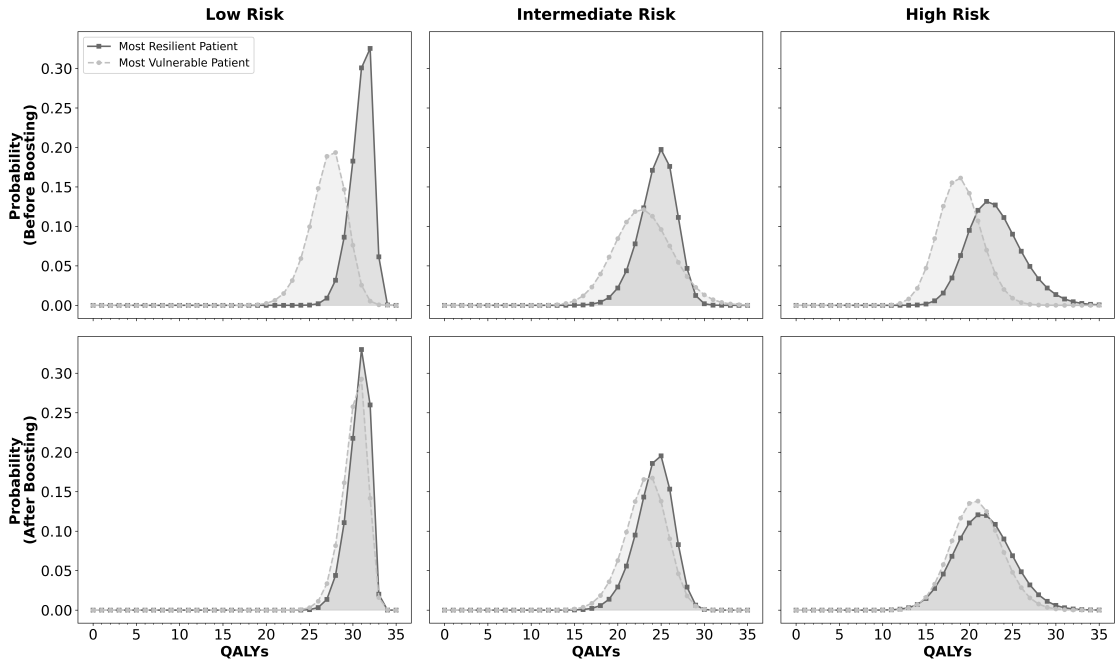}
    \caption{\textbf{\boldmath Learned QALYs distributions of the two most different patients in each group.}}
    \label{fig:Learned Distribution Plots}
\end{figure}

Table~\ref{tab:baseline_comparison_groups} compares the cumulative discounted QALYs of the proposed \gls{BDRL} algorithm against standard \gls{RL} baselines. Results are stratified by risk cluster (low, intermediate, high) and further partitioned into the most resilient, median, and the most vulnerable patients within each cohort. Compared with \gls{DRL}, Deep Q-Network and Q-learning, \gls{BDRL} achieves comparable or higher \glspl{qaly} across all three clusters while preserving performance for the most resilient patients.

\begin{table}[h!]
    \centering
    \caption{Comparison of the most resilient, median, and most vulnerable patients across risk groups.}
    \label{tab:baseline_comparison_groups}
    
    \setlength{\tabcolsep}{6pt}
    \begin{tabular}{llccc}
        \toprule
        \textbf{Algorithm} & \textbf{Risk Group} & \textbf{Most Resilient} & \textbf{Median} & \textbf{Most Vulnerable} \\
        \midrule
        \multirow{3}{*}{BDRL (Base)} 
                          & Low    & 32.35 & 32.19 & 31.82 \\
                          & Intermediate & 26.63 & 26.35 & 26.26 \\
                          & High   & 19.78 & 19.63 & 19.37 \\
        \midrule
        \multirow{3}{*}{DRL ($\lambda=0$)} 
                          & Low    & 32.35 & 31.62 & 29.75 \\
                          & Intermediate & 26.63 & 26.08 & 24.46 \\
                          & High   & 19.78 & 18.64 & 17.06 \\
        \midrule
        \multirow{3}{*}{Deep Q-Network} 
                          & Low    & 31.15 & 29.84 & 28.18 \\
                          & Intermediate & 25.11 & 24.05 & 22.47 \\
                          & High   & 17.82 & 16.27 & 14.92 \\
        \midrule
        \multirow{3}{*}{Q-learning} 
                          & Low    & 27.44 & 26.12 & 24.30 \\
                          & Intermediate & 22.57 & 21.09 & 19.43 \\
                          & High   & 15.54 & 14.18 & 11.98 \\
        \bottomrule
    \end{tabular}
\end{table}

\subsubsection{Sensitivity Analyses.}
To assess the robustness of the observed health outcomes, we conduct a comprehensive sensitivity analysis across the algorithm's key hyperparameters. 

\paragraph{Return maximization and constraint satisfaction trade-off.} A critical component of the \gls{BDRL} framework is the Lagrangian multiplier $\lambda$, which governs the trade-off between return maximization and constraint satisfaction. We analyze the sensitivity of patient \glspl{qaly} to variations in $\lambda$, ranging from slightly constrained to highly regularized regimes. As shown in Figure~\ref{fig:three_panels_vertical}, a moderate $\lambda$ effectively enhances the \glspl{qaly} of median and vulnerable patients, whereas excessive values over-constrain the optimization, causing the penalty term to dominate the return signal.

\begin{figure}[htbp]
    \centering
    
    \begin{subfigure}{\textwidth}
        \centering
        \includegraphics[width=\linewidth]{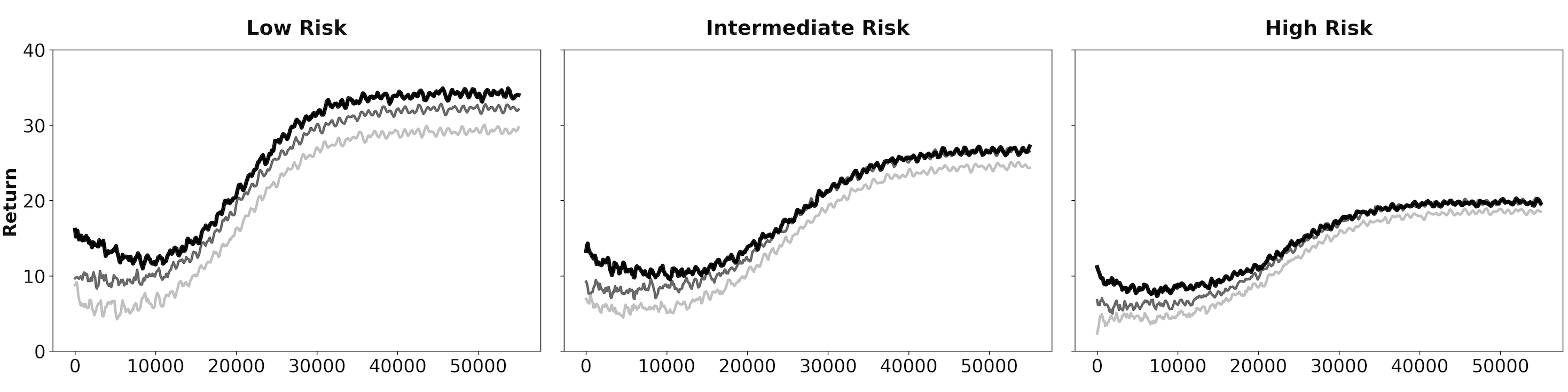}
        \caption{$\lambda = 0.01$}
        \label{fig:panel_a}
    \end{subfigure}
    
    
    \begin{subfigure}{\textwidth}
        \centering
        \includegraphics[width=\linewidth]{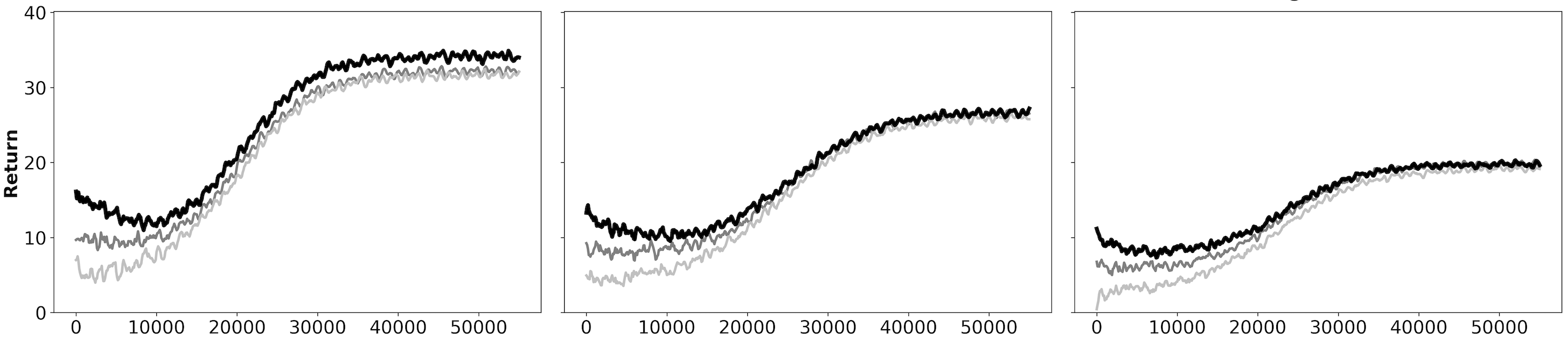}
        \caption{$\lambda = 0.1$}
        \label{fig:panel_b}
    \end{subfigure}
    
    
    \begin{subfigure}{\textwidth}
        \centering
        \includegraphics[width=\linewidth]{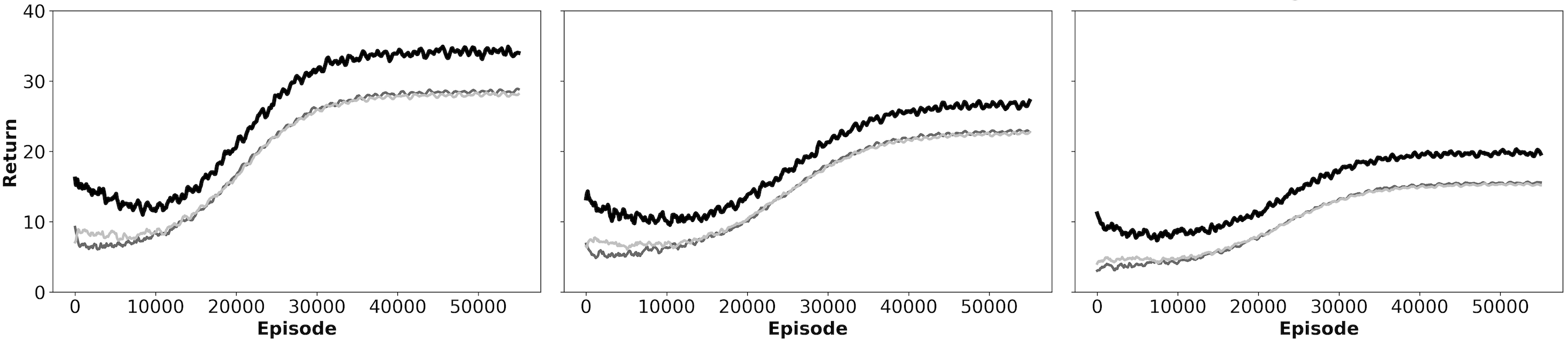}
        \caption{$\lambda = 0.5$}
        \label{fig:panel_c}
    \end{subfigure}
    
    \caption{\textbf{\boldmath Comparison of QALYs across different penalty weights.}}
    \label{fig:three_panels_vertical}
\end{figure}

\paragraph{Return difference tolerance, mixing parameter, and fallback weight.} Beyond $\lambda$, we further evaluate the stability of the solution under varying model parameters, including the projection tolerance $\epsilon$, the fallback weight $\rho$, and the mixing parameter $\alpha$. Table~\ref{tab:all_sensitivity_results} presents our results. Overall, the \gls{BDRL} framework remains robust across various hyperparameter settings, though appropriate tuning can further enhance algorithm performance.

\begin{table}[h!]
\centering
\caption{\textbf{Mean discounted QALYs.} The base case uses $W_2$-regularization with $\lambda=0.1$, $\epsilon=0.010$, $\rho=0.90$, $\alpha=0.05$. Sensitivity rows vary one parameter at a time. The \gls{KL} and \gls{JS} rows show results across regularization strengths $\lambda\in\{0,0.01,0.10,0.50\}$.}
\label{tab:all_sensitivity_results}
\begin{tabular}{ll ccc ccc ccc}
\toprule
& & \multicolumn{3}{c}{\textbf{Low risk}}
  & \multicolumn{3}{c}{\textbf{Intermediate}}
  & \multicolumn{3}{c}{\textbf{High risk}} \\
\cmidrule(lr){3-5}\cmidrule(lr){6-8}\cmidrule(lr){9-11}
\multicolumn{2}{c}{\textbf{Scenario}}
  & Resi. & Med. & Vulne.
  & Resi. & Med. & Vulne.
  & Resi. & Med. & Vulne. \\
\midrule
\multicolumn{2}{l}{\textit{Base case}}
              & 32.35 & 32.19 & 31.82 & 26.63 & 26.35 & 26.26 & 19.78 & 19.63 & 19.37 \\
\addlinespace
\multirow{2}{*}{$\epsilon$}
  & 0.005 & 32.08 & 31.52 & 31.03 & 26.46 & 26.12 & 25.91 & 19.25 & 18.78 & 18.57 \\
  & 0.050 & 32.19 & 31.68 & 31.25 & 26.41 & 26.03 & 25.76 & 19.34 & 19.08 & 18.87 \\
\addlinespace
\multirow{2}{*}{$\rho$}
  & 0.80  & 31.95 & 31.43 & 31.11 & 26.37 & 26.09 & 25.79 & 19.19 & 18.86 & 18.64 \\
  & 0.95  & 31.86 & 31.54 & 31.13 & 25.83 & 25.38 & 25.05 & 18.84 & 18.39 & 18.12 \\
\addlinespace
\multirow{2}{*}{$\alpha$}
  & 0.02  & 31.67 & 31.24 & 30.85 & 25.82 & 25.49 & 25.15 & 19.13 & 18.76 & 18.42 \\
  & 0.10  & 31.64 & 31.21 & 30.77 & 26.07 & 25.61 & 25.36 & 19.23 & 18.97 & 18.45 \\
\addlinespace
\multirow{4}{*}{KL}
  & 0     & 28.69 & 26.97 & 26.37 & 24.61 & 22.75 & 21.79 & 17.96 & 16.88 & 16.13 \\
  & 0.01  & 28.47 & 27.06 & 26.46 & 24.35 & 22.87 & 21.94 & 17.84 & 16.97 & 16.26 \\
  & 0.10  & 27.95 & 27.01 & 26.49 & 23.08 & 22.45 & 21.87 & 17.03 & 16.68 & 16.34 \\
  & 0.50  & 27.34 & 27.21 & 27.08 & 22.67 & 22.41 & 22.15 & 13.62 & 13.53 & 13.37 \\
\addlinespace
\multirow{4}{*}{JS}
  & 0     & 29.04 & 27.05 & 26.02 & 24.95 & 23.14 & 22.09 & 18.06 & 17.21 & 16.62 \\
  & 0.01  & 28.82 & 27.13 & 26.20 & 24.70 & 23.21 & 22.23 & 17.84 & 17.30 & 16.67 \\
  & 0.10  & 28.34 & 27.46 & 26.90 & 23.50 & 22.54 & 22.34 & 17.50 & 17.04 & 16.78 \\
  & 0.50  & 27.82 & 27.52 & 27.38 & 23.11 & 22.60 & 22.43 & 13.71 & 13.67 & 13.47 \\
\bottomrule
\end{tabular}
\end{table}

\paragraph{Alternative Distributional Constraints.} We also explore the \gls{KL} and \gls{JS} divergences to regularize the discrepancy between the return distributions. However, these metrics proved detrimental to the optimization process. For example, suppose two patients in the same risk group have predicted return distributions with disjoint supports: one distribution places most of its probability mass on outcomes near 10 \glspl{qaly}, while the other concentrates near 18 \glspl{qaly} due to initial random initialization. In this case, the \gls{KL} divergence becomes infinite (undefined) because the distributions have disjoint supports. By contrast, the \gls{JS} divergence saturates to a constant value, providing zero gradient signal to guide the distributions closer together. These pathologies yield unstable or uninformative gradients, which in turn require aggressive gradient clipping to prevent divergence. However, heavy clipping suppresses the update signal precisely where improvements are needed, limiting gains for the highest-risk patients. Consistent with this effect, Table~\ref{tab:all_sensitivity_results} shows substantially poorer clinical outcomes under divergence-based regularization: for the most vulnerable patient in each group, QALYs fall to as low as 13.37, markedly below those achieved by \gls{BDRL}.

\section{Discussion}\label{sec:Discussion}
This work introduced \gls{BDRL}, a new algorithm that improves the return distributions of low-performing agents by aligning their behavior with high-performing references. Our primary contributions included the formulation of a $W_2$-regularized loss function that ensures similar distributions among comparable agents and stable gradient signaling across disjoint supports, and the derivation of a convex quadratic post-update projection that guaranties outcome consistency without destabilizing the learning process. We showed the algorithmic contraction properties, stability, and efficiency under mild assumptions. 

Our empirical results demonstrated the practical value of \gls{BDRL}. We observed that our algorithm generates comparable policies for vulnerable and resilient agents, improving outcomes for vulnerable patients and those with median outcomes without compromising outcomes for resilient references. This merit resulted in policies that outperformed standard \gls{DRL}, Q-learning, and Deep Q-Network in terms of aggregate \glspl{qaly}.

Our results highlighted a consistent trade-off governed by the penalty weight $\lambda$. Moderate values of $\lambda$ can improve outcomes for patients experiencing subpar results while simultaneously reducing within-group dispersion, whereas excessively large $\lambda$ can over-constrain learning and degrade \glspl{qaly}. This pattern aligns with the intended role of $\lambda$: it mediates the balance between outcome optimality and constraint satisfaction. Extensive sensitivity analyses, varying the projection tolerance $\epsilon$, mixing parameter $\alpha$, and fallback weight $\rho$, demonstrated that the algorithm remained robust with respect to hyperparameter selection. 

A notable methodological insight is the importance of the choice of distributional discrepancy measure. Although \gls{KL} and \gls{JS} divergences are widely used in probabilistic learning, our experiments indicated that they were poorly suited for regularizing return distributions in our setting. The \gls{KL} divergence can become undefined or infinite when the supports are disjoint, while \gls{JS} divergence may saturate and yield vanishing gradients in similar regimes. In contrast, the $W_2$ distance leveraged the geometry of the outcome space and yields informative gradients even when two distributions have limited overlap, resulting in stable convergence and stronger performance. This finding suggested that transport-based regularization may be a more reliable metric for constraint-aware distributional \gls{RL} when distributional probability mass varied substantially across agents.

We concluded by discussing potential opportunities for future research from methodological and applied perspectives. From an algorithmic standpoint, our operationalization of comparability relies on discrete partitioning (e.g., grouping patients into $k=3$ clusters). While computationally efficient, this approach introduces dependency on the specific clustering method and granularity chosen. A natural extension is to replace discrete groups with continuous similarity weighting \citep{lahoti2019operationalizing}, enabling constraints that vary smoothly with patient similarity rather than relying on hard assignments. Furthermore, our constraints are currently enforced on learned return distributions. Although our theoretical analysis provides a bound connecting these learned constraints to ground-truth outcomes, finite-sample effects and approximation errors remain relevant practical challenges, particularly when trajectory counts are limited. Future methodological work could address these difficulties by developing adaptive strategies for tuning the penalty weight $\lambda$ during training, such as dual-ascent updates \citep{ding2020natural} or specific constraint violation targets, to reduce hyperparameter sensitivity and improve the robustness of the learned policy with limited samples. 

From a clinical perspective, our evaluation relies on a simulation-based approximation of patient dynamics. Consequently, the results inherit the assumptions embedded in the underlying risk models, treatment-effect estimates, and progression equations. Any misspecification in the simulator can propagate to the learned policies. Bridging simulation-based learning with observational data represents a critical next step to mitigate this risk. Future efforts should focus on integrating distribution shift diagnostics with online data to enhance the reliability of deploying \gls{RL} policies in real-world clinical decision support. Additionally, our current definition of comparability is based solely on baseline health conditions. Different patient partitions that incorporate additional comorbidities or social determinants of health could alter the constraint structure. To address this shortcoming, future work may extend the framework to incorporate explicit constraint functions beyond distributional similarity, such as safety thresholds, treatment burden limits, and monotonicity constraints with respect to specific risk factors.

Overall, our results showed that \gls{BDRL} provides a principled and effective mechanism for learning clinically beneficial policies while promoting comparable outcomes among similar agents. Beyond empirical gains, our methodological findings underscore the need to move beyond expectation-based objectives in high-stakes settings. We demonstrated that modeling the full return distribution is essential for capturing the heterogeneous risk profiles. Moreover, we showed that regularization via the 2-Wasserstein distance provides the numerical stability required to enforce constraints even when outcome probability masses are disjoint. By integrating distributional learning with geometry-aware constraints, our algorithm advances the usability of sequential decision-making techniques in healthcare settings where outcome quality and consistency are central.


%
%
%
%
%

%



\bibliographystyle{plainnat} 
\bibliography{sample}

\newpage
\setcounter{section}{0}
\appendix

\setcounter{figure}{0}
\makeatletter
\renewcommand{\fnum@figure}{Supplementary Figure~\thefigure}
\makeatother

\section{Theoretical Results}
\setcounter{theorem}{0}
\renewcommand{\thetheorem}{\arabic{theorem}}
\setcounter{proposition}{0}
\setcounter{property}{0}

\label{app:theory_results}
\begin{property}[Linear Mixture Contraction]
\label{prop:LinearMixtureContraction}
Let $Z$ be any distribution in $\mathcal{P}_2(\mathbb{R}^d)$ and let $Z^*$ be a target distribution. For any $\rho \in (0,1)$, define the mixture update:
\[
\widetilde{Z} = \rho Z + (1-\rho) Z^*.
\]
Then the update is a strict contraction toward $Z^*$ in the 2-Wasserstein metric:
\[
W_2(\widetilde{Z}, Z^*) \leq \sqrt{\rho} \, W_2(Z, Z^*).
\]
Consequently, repeated application of this update ensures that $W_2(\widetilde{Z}^{(n)}, Z^*) \to 0$ as $n \to \infty$.
\end{property}

\begin{proof}

\textit{Proof.} Let $\pi^* \in \Pi(Z, Z^*)$ be an optimal coupling between $Z$ and $Z^*$ such that:
\[
W_2^2(Z, Z^*) = \int_{\mathbb{R}^d \times \mathbb{R}^d} \|x - y\|^2 \, d\pi^*(x, y).
\]
We construct a candidate coupling $\tilde{\pi}$ between the mixture $\widetilde{Z}$ and the target $Z^*$ by taking a convex combination of the optimal coupling $\pi^*$ and the identity coupling $\text{Id}_{\# Z^*}$ (the coupling that maps $Z^*$ to itself):
\[
\tilde{\pi} = \rho \pi^* + (1-\rho) \text{Id}_{\# Z^*}.
\]

\noindent\textbf{1. Verification of Marginals:}
The first marginal of $\tilde{\pi}$ is:
\begin{align*}
\operatorname{Proj}_{1}(\tilde{\pi})
&= \rho\, \operatorname{Proj}_{1}(\pi^*)
   + (1-\rho)\, \operatorname{Proj}_{1}(\operatorname{Id}_{\# Z^*}) \\
&= \rho Z + (1-\rho) Z^* \\
&= \widetilde{Z}.
\end{align*}
The second marginal of $\tilde{\pi}$ is:
\begin{align*}
\operatorname{Proj}_{2}(\tilde{\pi})
&= \rho\, \operatorname{Proj}_{2}(\pi^*)
   + (1-\rho)\, \operatorname{Proj}_{2}(\operatorname{Id}_{\# Z^*}) \\
&= \rho Z^* + (1-\rho) Z^* \\
&= Z^*.
\end{align*}
Thus, $\tilde{\pi} \in \Pi(\widetilde{Z}, Z^*)$ is a valid coupling.

\noindent\textbf{2. Bounding the Distance:}
By the definition of the 2-Wasserstein distance as the infimum over all valid couplings, we have:
\begin{align*}
W_2^2(\widetilde{Z}, Z^*)
&\leq \int_{\mathbb{R}^d \times \mathbb{R}^d}
   \|x - y\|^2 \, d\tilde{\pi}(x, y) \\
&= \rho \int_{\mathbb{R}^d \times \mathbb{R}^d}
   \|x - y\|^2 \, d\pi^*(x, y) \\
&\quad + (1-\rho)
   \int_{\mathbb{R}^d \times \mathbb{R}^d}
   \|x - y\|^2 \, d\operatorname{Id}_{\# Z^*}(x, y).
\end{align*}
In the identity coupling $\text{Id}_{\# Z^*}$, the mass is concentrated on the diagonal where $x=y$, implying $\|x - y\|^2 = 0$. Therefore, the second integral vanishes, leaving:
\[
W_2^2(\widetilde{Z}, Z^*) \leq \rho W_2^2(Z, Z^*).
\]
Taking the square root of both sides yields $W_2(\widetilde{Z}, Z^*) \leq \sqrt{\rho} W_2(Z, Z^*)$. Since $\sqrt{\rho} < 1$, the distance strictly decreases at each step. Convergence to zero follows from the geometric decay $(\sqrt{\rho})^n W_2(Z^{(0)}, Z^*)$.
\end{proof}

\begin{theorem}[The Contraction of Mixtures]
\label{thm:Contraction of Mix}
The distance between the updated distributional estimate $Z_i^{(t)}$ and the optimal group estimate $Z_{g_k}^{*}$ is nonexpansive under the mixture operation. Specifically, let $d^{(t)}= W_2(Z_i^{(t)}, Z_{g_k}^{*})$. We claim that the distance at the subsequent step satisfies: 
\begin{equation*}
    d^{(t+1)} \leq \max \{d^{(t)}, \epsilon \},
\end{equation*}
where $\epsilon>0$ is a predefined error bound.
\end{theorem}

\begin{proof}
\noindent\textit{Proof.} Let $d^{(t)} := W_2(Z^{(t)}, Z^*)$ denote the distance to the target distribution at iteration $t$. At each iteration, the algorithm produces a candidate update $Z_i'$ and then determines the next iterate $Z^{(t+1)}$ according to one of the following mutually exclusive situations: (i) the candidate $Z_i'$ already lies within the $\epsilon$-neighborhood of $Z^*$;
(ii) although $Z_i'$ is not within the $\epsilon$-ball, there exists a convex combination of $Z^{(t)}$ and $Z_i'$ that enters the $\epsilon$-ball, in which case we select such a mixture as $Z^{(t+1)}$; or (iii) no convex combination of $Z^{(t)}$ and $Z_i'$ can satisfy the $\epsilon$ constraint, so we apply the fallback contraction update $Z^{(t+1)}=\rho Z^{(t)}+(1-\rho)Z^*$. We now verify in each case that $d^{(t+1)} \le \max\{d^{(t)},\epsilon\}$.

Case 1 ($Z_i^{'}$ is already close): If $W_2(Z_i^{'},Z_{g_k}^{*}) \leq \epsilon$, then we have: 
\begin{equation*}
    d^{(t+1)} = W_2(Z_i^{t+1},Z_{g_k}^{*}) 
= W_2(Z_i^{'},Z_{g_k}^{*}) 
\leq \epsilon \leq \max\{d^{(t)}, \epsilon \}.
\end{equation*}

Case 2 (there is a convex combination within $\epsilon$): We obtain $\alpha^* \in (0,1)$ by solving Problem 1 in Algorithm 2, so that: $W_2(\alpha^* Z_i^{t} + (1-\alpha^*)Z_i^{'}, Z_{g_k}^{*}) \leq \epsilon$.
The next estimate $Z^{(t+1)}$ is set to $\alpha^* Z_i + (1-\alpha^*)Z_i'$, so $d^{(t+1)} \le \epsilon \le \max\{d^{(t)}, \epsilon\}$.

Case 3 (fallback contraction toward $Z^*$): No convex combination with $Z'$ stays within the $\epsilon$-ball, so we apply
\[
Z^{(t+1)}
~=~ \rho\,Z^{(t)} \;+\;(1-\rho)\,Z^*,
\]
with $0<\rho<1$. By Property~\ref{prop:LinearMixtureContraction}, we have
\begin{align*}
d^{(t+1)}
&= W_2\bigl(\rho\,Z^{(t)} + (1-\rho)\,Z^*,\, Z^*\bigr) \\
&\le \sqrt{\rho}\, W_2\bigl(Z^{(t)}, Z^*\bigr) \\
&< W_2\bigl(Z^{(t)}, Z^*\bigr) \\
&= d^{(t)}.
\end{align*}
If $d^{(t)}>\epsilon$, we have
\[
d^{(t+1)} \;\le\; d^{(t)} \;=\;\max\{d^{(t)},\,\epsilon\},
\]
If $d^{(t)}<=\epsilon$, we have
\[
d^{(t+1)} \;\le\; d^{(t)} \;\leq\;\max\{d^{(t)},\,\epsilon\} = \epsilon\,
\]
completing the proof that $d^{(t+1)}\le\max\{d^{(t)},\epsilon\}$ in this branch.
\end{proof}

\begin{proposition}[$W_2$ Distance Convergence]
\label{prop:Nonexpansive}
The update rule guarantees that the $W_2$ distance between the current estimate and the optimal estimate remains nonexpansive, leveraging the property established in Theorem~\ref{thm:Contraction of Mix}. Consequently, the eventual $W_2$ distance will stay within the $\epsilon$ range.
\end{proposition}

\begin{proof}
\noindent\textit{Proof.} Whenever $d^{(t)} > \epsilon$, Case 3 above yields: \\
$d^{(t)} := W_2(Z_i^{(t)}, Z_{g_k}^*)$ and define the mixture family
\[
\tilde Z(\alpha') := \alpha' Z_i' + (1-\alpha') Z_i^{(t)}, \qquad \alpha'\in[0,1].
\]
By definition of $d^{(t+1)}$ as the best mixture (projection) onto the feasible set,
\begin{equation}
\label{eq:dtplus1-min}
d^{(t+1)} \;=\; \min_{\alpha'\in[0,1]} W_2\!\bigl(\tilde Z(\alpha'),\, Z_{g_k}^*\bigr).
\end{equation}
In particular, since $\alpha'=1$ is a feasible choice, based on Property~\ref{prop:LinearMixtureContraction}, we have:
\begin{equation}
\label{eq:min-leq-endpoint}
d^{(t+1)}
\;\le\;
W_2\!\bigl(\tilde Z(1), Z_{g_k}^*\bigr)
=
W_2(Z_i', Z_{g_k}^*)
\;\le\; \sqrt{\rho}d^{(t)}
\end{equation}

Let $N = \min\{n: \rho^{n/2} d^{(0)} \leq \epsilon\} = \left\lceil \ln(\epsilon^2/d^{(0)^2})/\ln\rho \right\rceil$. This indicates after $N$ steps, we are guaranteed to enter the $\epsilon$ range. Note that $\rho$ can be arbitrarily small.
\end{proof}



\begin{theorem}[Convex Quadratic Transformation]
\label{thm:Convex Opti}
The conditional optimization problem defined by:
\begin{align*}
& \min_{\alpha} \ W_2(\alpha Z_i(s,a) + (1-\alpha)Z_i'(s,a),Z_i'(s,a)) \\
& \text{s.t.} \ \ W_2(\alpha Z_i(s,a) + (1-\alpha)Z_i'(s,a),Z_i^{*}(s,a)) \leq \epsilon \\
& \ \ \ \ \ \ \alpha \in (0,1)
\end{align*}
can be reformulated as a convex quadratic optimization problem. This transformation allows for efficient numerical solutions using standard quadratic programming solvers, substantially improving the computational efficiency of the training loop.
\end{theorem}

\begin{proof}
\textit{Proof.} We consider discrete approximations of the return distributions supported on a common finite grid, which is also the condition of our \gls{BDRL} algorithm.
Here:
\begin{itemize}
    \item $D$ denotes the number of support points (quantile atoms) in the discretization;
    \item $\{z_d\}_{d=1}^D$ are the fixed support locations shared by all three distributions;
    \item $\delta_{z_d}$ denotes the Dirac delta (point mass) at location $z_d$;
    \item $p_d$, $p_d'$, and $q_d$ are the probability masses assigned to $z_d$ by $Z_i(s,a)$, $Z_i'(s,a)$, and $Z_i^*(s,a)$ respectively, satisfying
    \[
    p_d,\, p_d',\, q_d \ge 0, 
    \qquad 
    \sum_{d=1}^D p_d = \sum_{d=1}^D p_d' = \sum_{d=1}^D q_d = 1.
    \]

\end{itemize}
Let the distributions be represented by point masses $Z_i(s,a) = \sum\limits_{d=1}^{D} p_d \delta_{z_d}$, $Z_i'(s,a) = \sum\limits_{d=1}^{D} p_d' \delta_{z_d}$, and $Z_i^*(s,a) = \sum\limits_{d=1}^{D} q_d \delta_{z_d}$. \\

Let 
\[
u_{\alpha} = \alpha Z_i(s,a) + (1-\alpha)Z_i'(s,a),
\]
so that its point-mass representation is
\[
u_{\alpha} = \sum_{d=1}^D p_d(\alpha)\,\delta_{z_d},
\qquad
\text{with } 
p_d(\alpha) = \alpha p_d + (1-\alpha)p_d'.
\]

Because the cumulative distribution function (CDF) is the cumulative sum of these masses over the ordered support $\{z_d\}_{d=1}^D$, we have
\[
F_{u_{\alpha}}(z_d)
= \sum_{k=1}^d p_k(\alpha)
= \sum_{k=1}^d \bigl[\alpha p_k + (1-\alpha)p_k'\bigr].
\]
By linearity of summation, this can be written as
\begin{align*}
F_{u_{\alpha}}(z_d)
&= \alpha \sum_{k=1}^d p_k
  + (1-\alpha)\sum_{k=1}^d p_k' \\
&= \alpha F_{Z_i}(z_d)
  + (1-\alpha)F_{Z_i'}(z_d).
\end{align*}

Hence, the mixture of probability masses induces a corresponding linear mixture of the CDFs.
\\

The squared $W_2$ distance between two discrete distributions $\mu$ and $\nu$ on a common set of ordered points $\{z_d\}$ is given by $W_2^2(\mu, \nu) = \sum_{d} (z_d - z_{d-1})(F_{\mu}(z_d) - F_{\nu}(z_d))^2$ or a similar form depending on the measure. Since in our algorithm, we have a uniform step size $\Delta z$ between points, the objective function is:
$$W_2^2(u_{\alpha}, Z_i'(s,a)) = (\Delta z)^2 \sum\limits_{d=1}^{D-1} (F_{u_{\alpha}}(z_d) - F_{Z_i'}(z_d))^2$$
$$= (\Delta z)^2 \sum\limits_{d=1}^{D-1} \bigl(\alpha F_{Z_i}(z_d) + (1-\alpha)F_{Z_i'}(z_d) - F_{Z_i'}(z_d)\bigr)^2$$
$$= (\Delta z)^2 \alpha^2 \sum\limits_{d=1}^{D-1} (F_{Z_i}(z_d) - F_{Z_i'}(z_d))^2 = (\Delta z)^2 \alpha^2 A,$$
where $A = \sum\limits_{d=1}^{D-1} (F_{Z_i}(z_d) - F_{Z_i'}(z_d))^2$. Since $\alpha^2$ is convex and $A > 0$, the objective is convex.

Now, consider the constraint:
$$W_2^2(u_{\alpha}, Z_i^*(s,a)) = (\Delta z)^2 \sum\limits_{d=1}^{D-1} \bigl(F_{u_{\alpha}}(z_d) - F_{Z_i^*}(z_d)\bigr)^2 \leq \epsilon^2.$$
Let $a_d = F_{Z_i}(z_d) - F_{Z_i'}(z_d)$ and $b_d = F_{Z_i'}(z_d) - F_{Z_i^*}(z_d)$.
Then $F_{u_{\alpha}}(z_d) - F_{Z_i^{*}}(z_d) = \alpha a_d + b_d$.
The constraint becomes:
$$(\Delta z)^2 \sum\limits_{d=1}^{D-1} (\alpha a_d + b_d)^2 \leq \epsilon^2$$
$$(\Delta z)^2 \sum\limits_{d=1}^{D-1} (\alpha^2 a_d^2 + 2\alpha a_d b_d + b_d^2) \leq \epsilon^2$$
$$(\Delta z)^2 \left( \alpha^2 \sum\limits_{d=1}^{D-1} a_d^2 + 2\alpha \sum\limits_{d=1}^{D-1} a_d b_d + \sum\limits_{d=1}^{D-1} b_d^2 \right) \leq \epsilon^2$$
Substituting $A = \sum a_d^2$, $B = \sum a_d b_d$, and $C = \sum b_d^2$:
$$(\Delta z)^2 (A \alpha^2 + 2B\alpha + C) \leq \epsilon^2.$$

Since $\Delta z$ is a positive constant, dividing both sides of the inequality by $(\Delta z)^2$ preserves the ordering and simplifies the expression to:
\begin{align*}
& \min \ A \alpha^2 \\
& \text{s.t.} \ A \alpha^2 + 2B\alpha + C \leq \frac{\epsilon^2}{(\Delta z)^2} \\
& \ \ \ \ \ \alpha \in (0,1)
\end{align*}
The objective function is convex quadratic (since $A>0$), and the constraint is a convex inequality (a quadratic function of $\alpha$). Therefore, the problem is a convex quadratic optimization problem, which can be handled by standard solvers efficiently. Note that Property~\ref{prop:W2sq-convex} shows that the optimization problem in this theorem is strictly convex, even when formulated in a continuous space rather than a discrete one.

\end{proof}

\begin{property}[Convexity of Squared $W_2$ Distance]
\label{prop:W2sq-convex}
Fix $\nu\in\mathcal P_2(\mathbb R^d)$. Then the map
\[
f(\mu) := W_2^2(\mu,\nu)
\]
is convex in $\mu$ on $\mathcal P_2(\mathbb R^d)$.

\vspace{0.2cm}
\noindent\textbf{Significance.} 
This property establishes the theoretical foundation for the tractable projection step derived in Theorem \ref{thm:Convex Opti}. By proving convexity in the general space of probability measures, we guarantee that the underlying optimization landscape is strictly convex and free of local minima. From an engineering perspective, this ensures that our discrete implementation is mathematically well-posed and that standard quadratic programming solvers will reliably converge to the unique global optimum.
\end{property}

\begin{proof}
\textit{Proof.} \textbf{Step 0 (Couplings).}
Recall that $\Pi(\mu,\nu)$ denotes the set of all couplings of $\mu$ and $\nu$, i.e., all probability
measures $\pi$ on $\mathbb R^d\times\mathbb R^d$ whose first marginal is $\mu$ and second marginal is $\nu$:
\begin{align*}
\Pi(\mu,\nu)
&=
\Big\{
\pi:\ \pi(A\times \mathbb R^d)=\mu(A), \\
&\qquad\ \pi(\mathbb R^d\times B)=\nu(B),
\ \forall \text{ Borel } A,B
\Big\}.
\end{align*}
By definition,
\[
W_2^2(\mu,\nu) \;=\; \inf_{\pi\in\Pi(\mu,\nu)} \int \|x-y\|^2\,\pi(dx,dy).
\]

\noindent\textbf{Step 1 (Existence of optimal couplings).}
Since $\mu_1,\mu_2,\nu\in\mathcal P_2(\mathbb R^d)$ and the cost $c(x,y)=\|x-y\|^2$ is lower semicontinuous
and bounded from below, the Kantorovich problem admits an optimizer \citep{villani2009optimal}.
Hence, for $k\in\{1,2\}$, there exists an optimal coupling $\pi_k\in\Pi(\mu_k,\nu)$ such that
\[
W_2^2(\mu_k,\nu) \;=\; \int \|x-y\|^2\,\pi_k(dx,dy).
\]

\noindent\textbf{Step 2 (Mixture of couplings is a coupling).}
Fix $t\in[0,1]$ and define the mixture measure
\[
\pi := t\pi_1 + (1-t)\pi_2.
\]
We claim that $\pi \in \Pi(t\mu_1+(1-t)\mu_2,\nu)$.
Indeed, for any Borel sets $A,B\subseteq\mathbb R^d$, by linearity of measures,
\begin{align*}
\pi(A\times \mathbb R^d)
&= t\,\pi_1(A\times \mathbb R^d)
   + (1-t)\,\pi_2(A\times \mathbb R^d) \\
&= t\,\mu_1(A) + (1-t)\,\mu_2(A).
\end{align*}
and similarly,
\begin{align*}
\pi(\mathbb R^d\times B)
&= t\,\pi_1(\mathbb R^d\times B)
   + (1-t)\,\pi_2(\mathbb R^d\times B) \\
&= t\,\nu(B) + (1-t)\,\nu(B) \\
&= \nu(B).
\end{align*}
Therefore $\pi$ has the required marginals and is a valid coupling. This argument is in \cite{villani2009optimal}.

\noindent\textbf{Step 3 (Convexity).}
By linearity of integration in the measure,
\begin{align*}
\int \|x-y\|^2\,\pi(dx,dy)
&= t \int \|x-y\|^2\,\pi_1(dx,dy) \\
&\quad + (1-t)\int \|x-y\|^2\,\pi_2(dx,dy) \\
&= t\,W_2^2(\mu_1,\nu) + (1-t)\,W_2^2(\mu_2,\nu).
\end{align*}
Since $\pi$ is feasible for the transport problem between $t\mu_1+(1-t)\mu_2$ and $\nu$, we have
\begin{align*}
W_2^2\bigl(t\mu_1 + (1-t)\mu_2, \nu\bigr)
&= \inf_{\pi' \in \Pi(t\mu_1 + (1-t)\mu_2, \nu)} \\
&\qquad \int \|x-y\|^2\,\pi'(dx,dy) \\
&\le \int \|x-y\|^2\,\pi(dx,dy).
\end{align*}
and the desired convexity inequality follows:
\[
W_2^2(t\mu_1+(1-t)\mu_2,\nu)
\le
t\,W_2^2(\mu_1,\nu) + (1-t)\,W_2^2(\mu_2,\nu).
\]
\end{proof}

\FloatBarrier
\section{Cluster Determination}\label{app:clustering}
We determine the optimal number of clusters using the elbow method, selecting $k=3$ as the objective function showed diminishing loss improvement beyond this point (Supplementary Figure~\ref{fig:kmeans}). The inertia measure is within-cluster sum of squared distances. We use variables such as demographic information, blood pressure measurements, and diabetes status for clustering agents.

\begin{figure}[H]
    \centering
    \includegraphics[width=0.8\textwidth]{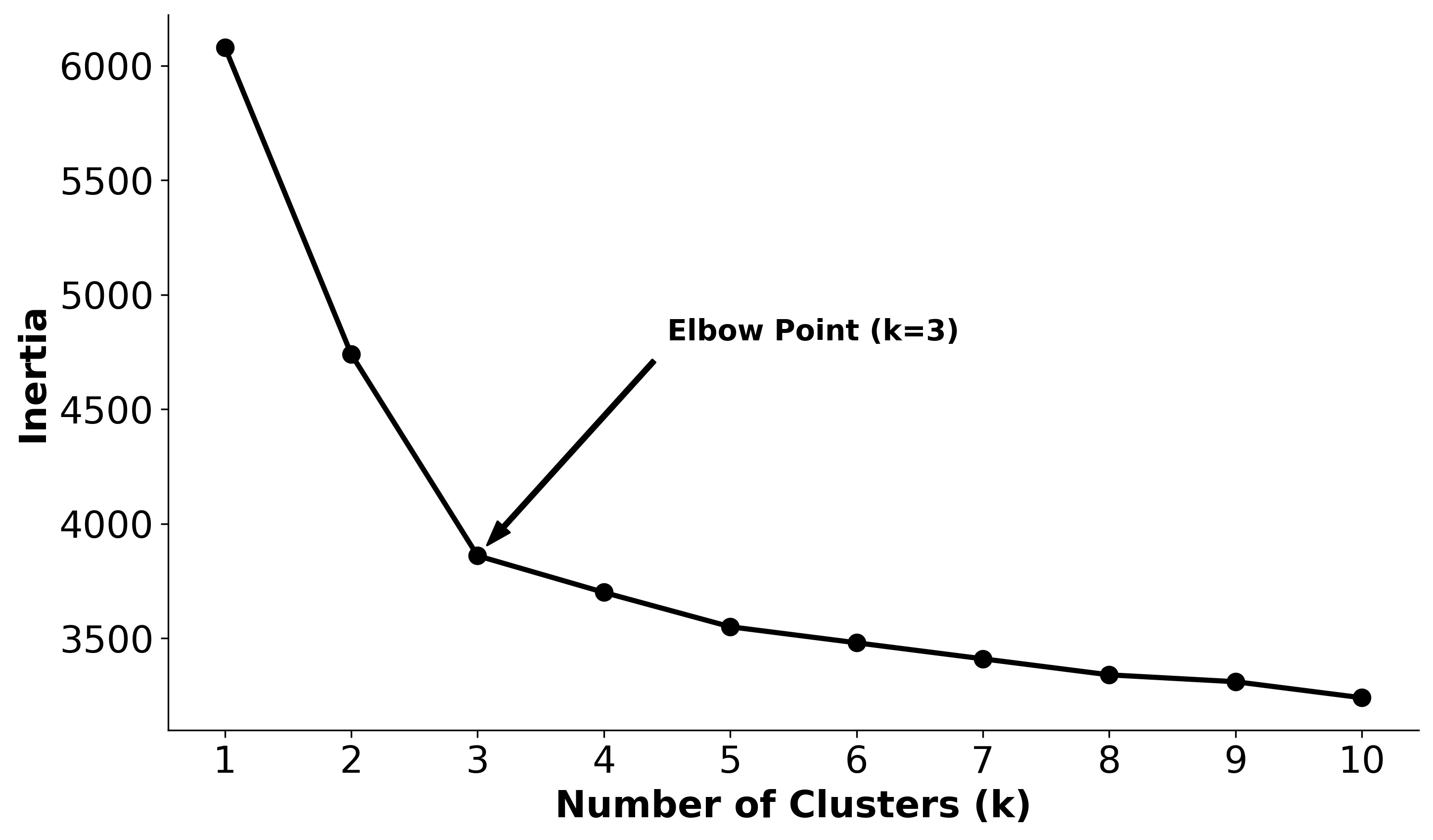}
    \caption{Selecting the number of centroids ($k$).}
    \label{fig:kmeans}
\end{figure}

\section{Batch Determination}
\label{app:convergence_diagnostics}

To ensure the statistical reliability of the policy evaluation and to construct valid confidence intervals for the estimated discounted QALYs, we employ the Batch Means Method. This approach is utilized to estimate the variance of the performance estimator and to verify that the number of simulation trajectories is sufficient for the central limit theorem to hold, thereby mitigating the risk of underestimating the standard error.

\subsection{Methodology}
We conducted a total of $N = 55,000$ independent Monte Carlo simulation trajectories for the evaluation of the optimal policy. To estimate the variance of the mean value function, the output sequences are partitioned into $B$ adjacent, non-overlapping batches, each of size $M = N/B$.

Let $Y_{j}$ denote the cumulative discounted QALYs obtained in the $j$-th simulation trajectory. The sample mean of the $q$-th batch, $\bar{Y}_q$, is calculated as:
\begin{equation*}
    \bar{Y}_q = \frac{1}{M} \sum_{j=(q-1)M + 1}^{qM} Y_{j}, \quad \text{for } q = 1, \dots, K
\end{equation*}

The grand mean estimator $\bar{\mu}$, representing the reported policy value, is the average of these batch means:
\begin{equation*}
    \bar{\mu} = \frac{1}{B} \sum_{k=1}^{B} \bar{Y}_q
\end{equation*}

The standard error (SE) of the estimator is then derived from the sample variance of the batch means:
\begin{equation*}
    \widehat{SE}(\bar{\mu}) = \sqrt{\frac{1}{B(B-1)} \sum_{q=1}^{B} (\bar{Y}_q - \bar{\mu})^2}
\end{equation*}

\subsection{Determination of Batch Count ($B$)}
To determine the optimal number of batches, we perform a sensitivity analysis on $B$. We vary the batch count over the range $B \in [10, 100]$ and monitor the stability of the estimated standard error $\widehat{SE}(\bar{\mu})$. The plot is shown in Supplemental Figure~\ref{fig:batch size}. We find $B=30$ to be the optimal choice because it yields enough batches to stabilize the standard error estimator while keeping individual batch sizes large enough to satisfy the Central Limit Theorem. This selection reaches the ``stability elbow" in standard error instability, without violating the normality assumptions required for valid confidence intervals of mean Shapiro-Wilk $p$-values.

\begin{figure}[H]
    \centering
    \includegraphics[width=0.7\textwidth]{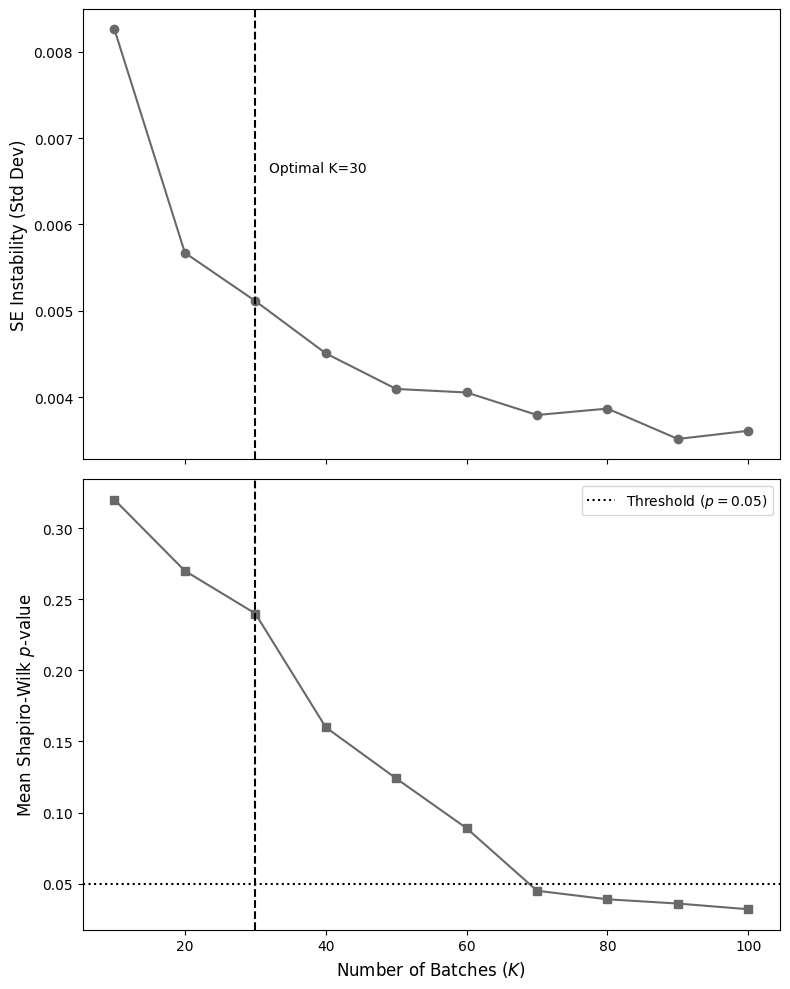}
    \caption{Finding the optimal number of batches ($B$).}
    \label{fig:batch size}
\end{figure}

\FloatBarrier
\section{Return Distribution Convergence}
\label{app:convergence_return_dist}
This section illustrates the convergence of return distributions across patients within the same risk group by tracking the trajectory of the $W_2$ distance between its two most dissimilar patients throughout the learning process. 

Supplementary Figure~\ref{fig:w2_convergence} displays a rapid decrease in distributional discrepancy across all subgroups during the first 10,000 episodes, indicating that the algorithm effectively prioritizes alignment early in the training phase. Notably, patients with high risk (darkest line) exhibit the highest initial disparity ($W_2 \approx 0.6$), reflecting the greater inherent stochasticity in high-risk trajectories, yet successfully converge to a stable low-error state ($W_2 < 0.1$) comparable to the low-risk and intermediate-risk groups by the end of training. Furthermore, a comparison of the final learned distributions reveals a substantial overlap, confirming that the model successfully mitigates distributional inconsistencies.

\begin{figure}[h!]
    \centering
    \includegraphics[width=0.7\linewidth]{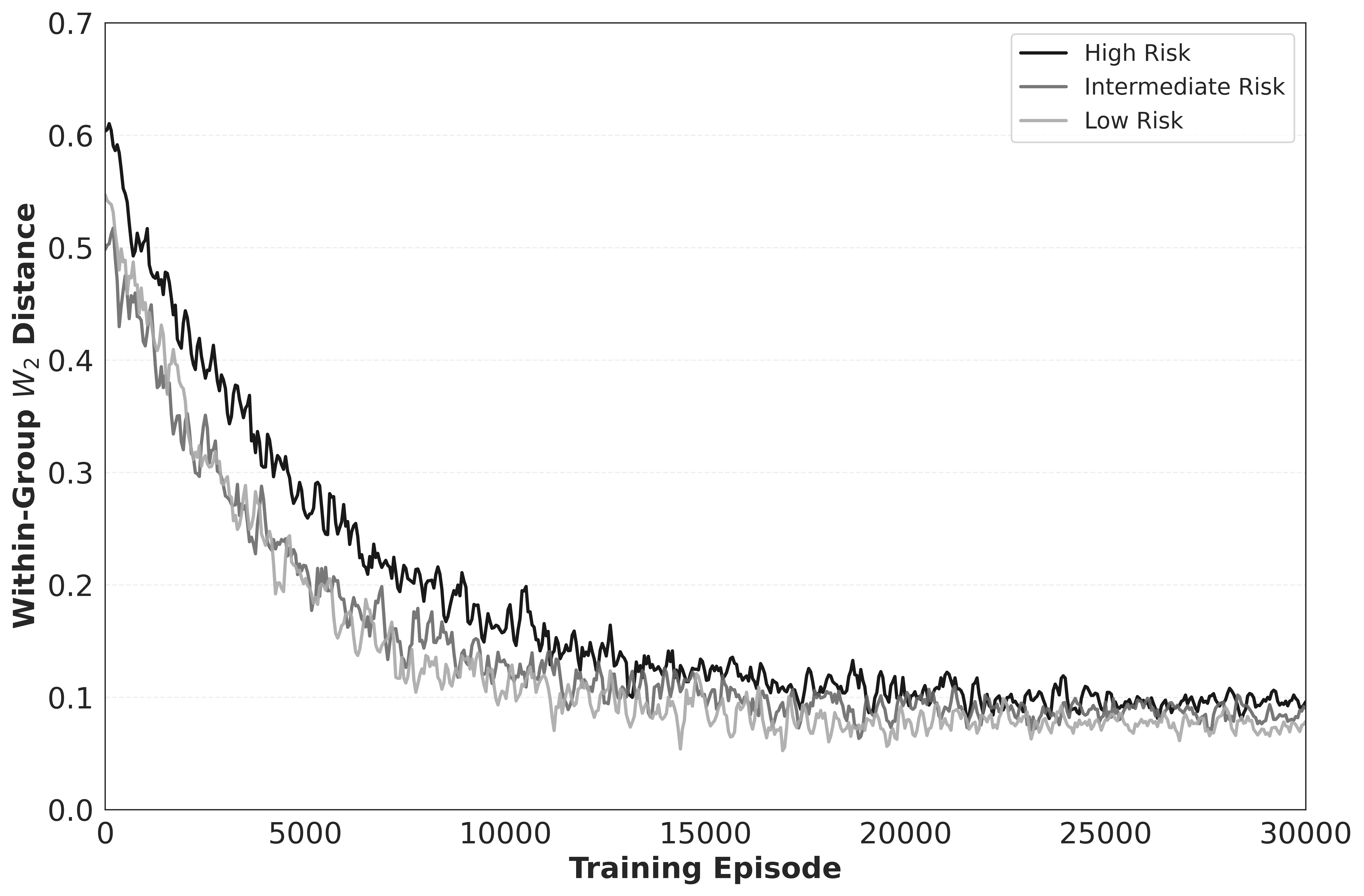}
    \caption{\textbf{\boldmath Convergence of $W_2$ distance within subgroups, showing the reduction in disparity over training episodes.}}
    \label{fig:w2_convergence}
\end{figure}

\end{document}